\documentclass{article}

\usepackage[dblblindworkshop, final]{neurips_2026}

\workshoptitle{I Can't Believe It's Not Better (ICBINB): Failure Modes of AI in Biology}

\usepackage[utf8]{inputenc} 
\usepackage[T1]{fontenc}    
\usepackage{hyperref}       
\usepackage{url}            
\usepackage{booktabs}       
\usepackage{amsfonts}       
\usepackage{amssymb}        
\usepackage{nicefrac}       
\usepackage{microtype}      
\usepackage{xcolor}         
\usepackage[normalem]{ulem} 

\usepackage{amsmath}
\usepackage{amsthm}
\newtheorem{proposition}{Proposition}
\usepackage{graphicx}
\usepackage{float}    
\title{A Generalisation Signal Need Not Be a Model-Selection Signal}

\author{%
  Aditya Nagarsekar\textsuperscript{1} \quad
  M P Ashish Bhat\textsuperscript{1} \quad
  Aadi Nesarkar\textsuperscript{1} \quad
  Vrishti Godhwani\textsuperscript{1} \\
  \textbf{Rahul Yedida\textsuperscript{2} \quad
  Aditya Challa\textsuperscript{1} \quad
  Sravan Danda\textsuperscript{1} \quad
  Snehanshu Saha\textsuperscript{1,3}} \\[6pt]
  $^{1}$Department of CS\&IS, BITS Pilani, K K Birla Goa Campus \\
  \texttt{\{f20230473,f20231146,f20210967,f20220260,} \\
  \texttt{adityac,dandas,snehanshus\}@goa.bits-pilani.ac.in} \\
  $^{2}$LexisNexis Legal \& Professional \quad $^{3}$Center for AI and Supercomputing, Mahindra University \\
  \texttt{rahul@ryedida.me}, \texttt{snehanshu.saha@mahindrauniversity.edu.in}
}

\begin{document}

\maketitle

\begin{abstract}
Model selection in computational biology often relies on validation data drawn from
the training regime, even when deployment lies outside it. When validation no
longer preserves which model is best, a natural alternative is to rank candidates
using properties of the trained network itself. We test this idea using a novel, 
forward-only proxy motivated by the norm of the Hessian, alongside common
Hessian measures, across molecular property, protein fitness, and drug-response
tasks. Contrary to our hypothesis, geometry does not become more useful as
validation Spearman correlation deteriorates: augmenting validation helps some shifts but significantly harms others. More surprisingly, the proxy still correlates
with generalisation gap on most tasks even when Hessian trace and top-eigenvalue relationships are
weak or reversed, yet this signal does not reliably identify the deployment-best
model. A curvature bound need not preserve cross-model rankings, and low
geometric scores can even favour collapsed predictors. Thus, a generalisation
signal need not be a model-selection signal.
\end{abstract}

\section{Introduction and Background}
 
Machine-learning models in biology are often selected on data that differ from
deployment: validated on known chemical series before use on new scaffolds, on
measured protein variants before searching distant regions of the fitness landscape,
or on familiar batches, donors or perturbations before new regimes. If candidates
$h = \{h_1,\ldots,h_K\}$ are selected by $h_{\mathrm{val}}=\arg\min_{h_i \in h} L_{\mathrm{val}}(h_i)$,
selection stays useful under shift only if validation approximately preserves their
ordering under the unobserved deployment loss $L_{\mathrm{dep}}$. Deployment can be
much harder without breaking this, but a shift that changes which model is best
breaks it, however accurately validation is estimated.

Biological data routinely violate the assumptions behind random held-out evaluation,
through dependence, confounding, preprocessing leakage, and distributional structure
\citep{whalen2022navigating}. Protein-fitness benchmarks increasingly evaluate
position-, wild-type-, and fitness-based shifts that can reverse conclusions drawn
from simpler evaluations \citep{didi2026flip2}; drug-response models deteriorate
sharply on unseen compounds, with careful reevaluation shrinking apparent gains over
simple baselines \citep{bernett2025dreval}; and metric choice alone can change
rankings under unseen chemistry \citep{bisht2026metric}, motivating the partitioning
framework of \citet{fernandezdiaz2024new}. This is distinct from a separate failure:
validation is itself estimated, and repeatedly optimising against a finite set can
overfit at the level of model selection \citep{cawley2010overfitting,schneider2025overtuning}.
 
One method to overcome this is to select by properties of the trained network itself.
Flatness has long been tied to generalisation \citep{hochreiter1997flat,jiang2020fantastic},
and efficient SAM variants have reached molecular graph transformers
\citep{wang2024graphsam}. Yet the evidence that such training-only signals track
\emph{out-of-distribution} behaviour is mixed: no generalisation bound is uniformly
tight across distributions or algorithms \citep{gastpar2023fantastic}, sharper minima
can generalise better OOD \citep{andriushchenko2023modern}, sharpness-aware training
helps OOD generalisation without a full account of why \citep{schapiro2024sharpness},
and geometric measures respond to incidental training choices \citep{kaur2023hessian}.
Curvature has nevertheless been used to guide hyperparameter optimisation directly
\citep{flatnessguidedhpo,strongconvexityhpo}, alongside zero-cost proxies that rank
architectures without full training \citep{lukasik2023evaluation}. Concurrent work
finds that the predictive value of generalisation measures changes under generic
distribution shift \citep{nakai2026revisiting}, but evaluates correlation; we ask
whether such a signal selects better models once validation becomes unreliable,
including when it guides HPO under biological shift.

We separate two questions: whether a training-only signal
\emph{correlates} with generalisation, and whether it is \emph{usable for selecting}
among models that are already trained. Our central finding is that the first does not
imply the second {(Figure~\ref{fig:concept})}. We introduce a cheap forward-only activation proxy, motivated by a bound on the
Hessian Frobenius norm, and test whether it outperforms exact curvature summaries under
biological distribution shift, where held-out validation is most likely to fail.

\begin{figure}[t]
\centering
\includegraphics[width=\textwidth]{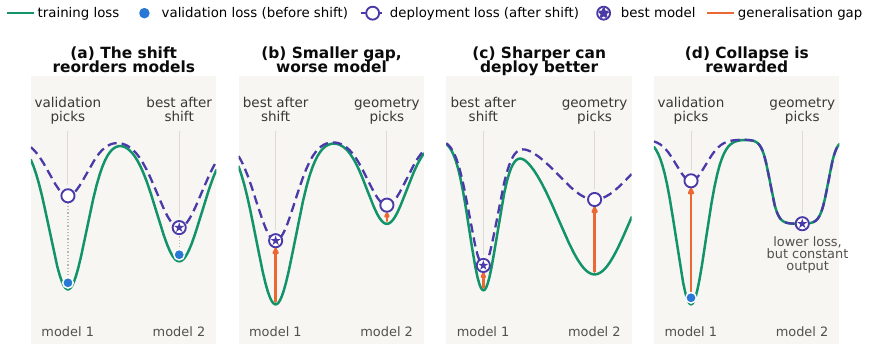}
\caption{{Schematic of the four failure modes studied in the paper, each shown on two
trained models. (a)~Validation's pick deploys worse after the shift; (b)~geometry picks the
smaller gap but the worse model; (c)~geometry picks the flatter minimum, but the sharper one
deploys better; (d)~geometry picks a collapsed network whose loss is lowest although its
output is constant. The curves are illustrative.}}
\label{fig:concept}
\end{figure}

\section{Problem and Methodology}

We study model selection under biological distribution shift. For a fixed candidate
pool, we measure \emph{rank transfer} as the Spearman correlation between validation
and deployment losses. We ask whether a training-only signal can replace
or augment validation specifically when this ranking deteriorates. Our primary generalisation signal is a novel, cheap activation-based proxy, motivated by a bound on the norm of the loss Hessian; \mbox{\S\ref{sec:failure}} compares it with full-network curvature and finds it correlates well with generalisation gap on most tasks, while common curvature measures do not. It is computed, with dropout off and on a fixed subset of training examples, from the post-activation matrix $A_{\ell,\mathcal B}$:
\begin{equation}
P(h)
=
\max_\mathcal B \frac{1}{L}\sum_{\ell=1}^{L}
\frac{\lVert A_{\ell,\mathcal B}\rVert_F^2}{|\mathcal B|\,w_\ell}
\end{equation}
The layer-averaged variant is the pre-registered primary proxy and is what
\texttt{Proxy} denotes throughout; penultimate-layer results are in the appendix. A formal derivation of $P(h)$ relating it to bounds on the loss Hessian norm is given in
Appendix~\mbox{\ref{app:proxy}}. As a mechanistic control, we compute
$\lambda_{\max}(\nabla_\theta^2L_{\mathrm{train}})$ over all weights and biases, without the $L_2$ term, using exact Hessian-vector products
\citep{pearlmutter1994fast}, following PyHessian
\citep{yao2020pyhessian}. Primary comparisons use the original candidate pools. Additional diagnostics are reported in the appendix.

\paragraph{{Hypothesis.}} {We hypothesise that as validation rank transfer deteriorates, model selection should benefit increasingly from curvature if flatter minima generalise better under shift, and our proxy if its generalisation-gap signal transfers to model ranking.}

We evaluate Caco2 and Lipophilicity from TDC \citep{huang2021tdc} under random,
scaffold, and constructed mismatch splits; FLIP2 Amylase and Hydrophobic Core
\citep{dallago2021flip,didi2026flip2}; and GDSC2 leave-drug-out
\citep{yang2013genomics,iorio2016landscape}. Molecular tasks use Morgan fingerprints
\citep{rogers2010ecfp}; protein tasks use one-hot sequence features; GDSC2 combines
molecular fingerprints with gene-expression features. Deployment MSE is the primary endpoint for statistical inference; benchmark-native and
rank-based metrics are reported as robustness checks.
A replication each derives a deterministic six-seed tuple from its
index, redrawing the split, the candidate pool, initialisation and batch ordering (full
counts and scheme in Appendix~\ref{app:protocol}). Within each replication we train the
same $K=48$ configurations sampled by scrambled Latin hypercube over learning rate,
weight decay, dropout, width, and depth, with no validation-based early stopping. All
selectors therefore rank the same models within a replication, which is what makes every
comparison paired. Each candidate is a ReLU MLP without normalisation layers, trained with Adam and coupled $L_2$ for 100 epochs. Augmented selectors (Val+X) deploy the candidate with the lowest sum of within-pool ranks on validation loss and X; remaining ties go to the lowest candidate index.
We additionally run 48-trial TPE and HEBO searches
\citep{bergstra2011algorithms,cowenrivers2022hebo}, comparing validation-only with
validation+geometry objectives. Paired tests use Wilcoxon signed-rank statistics
with Holm correction \citep{holm1979simple}; full search ranges and implementation
details are in the appendix.
\section{Results}
\label{sec:outcome}

OOD deployment does not necessarily break model selection. On the audited pools, rank transfer remains
moderate-to-strong across the molecular tasks ($\rho=0.44$--$0.77$), including the
constructed Lipophilicity mismatch ({$0.72$}), but nearly vanishes on {Amylase ($0.01$) and }Hydrophobic Core
($-0.10$) and is weak for GDSC2 leave-drug-out ($0.19$). Thus the relevant failure
is not OOD shift itself, but a shift that changes the candidate ordering.
Table~\ref{tab:all_results} shows that training-set geometry does not reliably
rescue this ranking failure.

\begin{table}[h]
\centering
\small
\setlength{\tabcolsep}{0pt}
\caption{Shared-pool deployment performance using each benchmark's native metric, after the post-hoc degeneracy audit. Proxy is the layer-averaged proxy. Mean$\pm$SD over outer replications is reported. Random is expected uniform selection. NDCG uses the full test set with no cutoff. Per-condition results under all metrics are in Appendix~\ref{app:pool}.}
\label{tab:all_results}
\begin{tabular*}{\textwidth}{@{\extracolsep{\fill}}llcccc@{}}
\toprule
Dataset & Metric & Val loss & Proxy & $\lambda_{\max}$ & Random \\
\midrule
Caco2-Wang (random) & MAE $\downarrow$ & 0.351$\pm$0.013 & 0.343$\pm$0.015 & 0.363$\pm$0.022 & 0.358$\pm$0.016 \\
Caco2-Wang (scaffold) & MAE $\downarrow$ & 0.432$\pm$0.047 & 0.421$\pm$0.040 & 0.442$\pm$0.038 & 0.443$\pm$0.039 \\
Lipophilicity (random) & MAE $\downarrow$ & 0.568$\pm$0.013 & 0.571$\pm$0.017 & 0.585$\pm$0.021 & 0.590$\pm$0.012 \\
Lipophilicity (scaffold) & MAE $\downarrow$ & 0.656$\pm$0.025 & 0.654$\pm$0.028 & 0.676$\pm$0.027 & 0.681$\pm$0.023 \\
Lipophilicity (mismatch) & MAE $\downarrow$ & 0.650$\pm$0.022 & 0.649$\pm$0.030 & 0.663$\pm$0.028 & 0.673$\pm$0.025 \\
FLIP2 Amylase & Spearman $\rho\uparrow$ & 0.011$\pm$0.103 & -0.010$\pm$0.115 & -0.062$\pm$0.104 & -0.015$\pm$0.028 \\
 & NDCG $\uparrow$ & 0.860$\pm$0.015 & 0.858$\pm$0.012 & 0.851$\pm$0.015 & 0.856$\pm$0.004 \\
FLIP2 Hydrophobic Core & Spearman $\rho\uparrow$ & 0.320$\pm$0.042 & 0.303$\pm$0.054 & 0.229$\pm$0.151 & 0.312$\pm$0.021 \\
 & NDCG $\uparrow$ & 0.908$\pm$0.014 & 0.904$\pm$0.012 & 0.896$\pm$0.026 & 0.909$\pm$0.004 \\
GDSC2 (leave-drug-out)& Pearson $r\uparrow$ & 0.485$\pm$0.101 & 0.480$\pm$0.099 & 0.477$\pm$0.100 & 0.480$\pm$0.096 \\
\bottomrule
\end{tabular*}

\end{table}

The proxy roughly matches validation on the molecular tasks, where rank
transfer is already moderate-to-strong, but does not improve as validation becomes
less reliable. On Hydrophobic Core it underperforms validation
($0.303$ vs.\ $0.320$ Spearman), while $\lambda_{\max}$ falls to $0.229$.
Under NDCG, FLIP2's own ranking metric, both geometric selectors are
nominally \emph{worse} than blind selection on Hydrophobic Core ($0.904$ and $0.896$
against $0.909$; unadjusted paired Wilcoxon $p=0.016$ and $p=0.022$).
This is genuine selection failure rather than lack of headroom: under MSE,
validation reaches $22.92$ while the within-pool oracle reaches $20.58$. Exploratory GDSC2 shows the same pattern: all tested selectors remain close to random despite
an oracle $8.4\%$ better than blind selection under MSE. A simple baseline using the number of model parameters unexpectedly beats validation on Hydrophobic Core under MSE, MAE, and
Spearman, while Amylase's apparent geometry-based MSE gain is caused by collapsed
predictors (Section~\ref{sec:failure}).

Steering HPO with geometry does not repair the mismatch either, as can be seen in Table~\ref{tab:seq_results}, which tests the stronger setting in which geometry actively
steers the search rather than only re-ranking a fixed pool. Proxy-guided search is largely neutral on the molecular tasks, with its clearest gain on Lipophilicity scaffold (TPE: $0.705$ vs.\ $0.688$; HEBO:
$0.701$ vs.\ $0.685$). The benefit does not appear where validation ranking is
weakest: on Hydrophobic Core TPE is unchanged and HEBO worsens
($0.273$ vs.\ $0.311$), while GDSC2 also deteriorates. The same signal can
help, do nothing, or hurt depending on the shift.

\begin{table}[h]
\centering
\small
\setlength{\tabcolsep}{0pt}
\caption{Sequential HPO deployment Spearman. Validation+Proxy denotes matched multi-objective search on the pre-registered layer-averaged proxy. Spearman is used throughout for a common scale-free comparison. Mean$\pm$SD over outer replications on unfiltered search trajectories; random search is the reference. Full per-arm results are in Appendix~\ref{app:seq}.}
\label{tab:seq_results}
\begin{tabular*}{\textwidth}{@{\extracolsep{\fill}}lccccc@{}}
\toprule
 & Random & \multicolumn{2}{c}{TPE} & \multicolumn{2}{c}{HEBO} \\
\cmidrule(lr){3-4}\cmidrule(lr){5-6}
Dataset & search & Val & Val+\,Proxy & Val & Val+\,Proxy \\
\midrule
Caco2-Wang (random) & 0.774$\pm$0.024 & 0.774$\pm$0.024 & 0.779$\pm$0.026 & 0.775$\pm$0.025 & 0.777$\pm$0.025 \\
Caco2-Wang (scaffold) & 0.688$\pm$0.085 & 0.684$\pm$0.099 & 0.684$\pm$0.102 & 0.678$\pm$0.089 & 0.698$\pm$0.080 \\
Lipophilicity (random) & 0.750$\pm$0.013 & 0.749$\pm$0.016 & 0.754$\pm$0.016 & 0.746$\pm$0.024 & 0.752$\pm$0.013 \\
Lipophilicity (scaffold) & 0.692$\pm$0.031 & 0.688$\pm$0.036 & 0.705$\pm$0.027 & 0.685$\pm$0.021 & 0.701$\pm$0.025 \\
Lipophilicity (mismatch) & 0.704$\pm$0.027 & 0.711$\pm$0.028 & 0.713$\pm$0.031 & 0.707$\pm$0.028 & 0.711$\pm$0.031 \\
FLIP2 Amylase & -0.007$\pm$0.081 & -0.053$\pm$0.092 & -0.054$\pm$0.096 & -0.028$\pm$0.102 & -0.043$\pm$0.079 \\
FLIP2 Hydrophobic Core & 0.336$\pm$0.063 & 0.300$\pm$0.046 & 0.302$\pm$0.042 & 0.311$\pm$0.035 & 0.273$\pm$0.036 \\
GDSC2 (leave-drug-out) & 0.432$\pm$0.084 & 0.426$\pm$0.097 & 0.414$\pm$0.074 & 0.429$\pm$0.087 & 0.412$\pm$0.072 \\
\bottomrule
\end{tabular*}
\end{table}

Under the MSE endpoint, on the original unfiltered pools, the primary
comparison (proxy versus validation, Holm-corrected across the four confirmatory
conditions) is significant on exactly one: Amylase, where the proxy improves deployment
MSE by $1.39$ and wins all 20 outer replications ($p_{\mathrm{Holm}}<0.001$).
Section~\ref{sec:failure} shows this is a degeneracy artifact; the other three are
non-significant ($p_{\mathrm{Holm}}=0.328$--$0.407$). Augmenting validation with the
proxy is significant in \emph{opposite directions}: it improves three confirmatory
conditions ($p_{\mathrm{Holm}}=0.027$--$0.047$) but is significantly \emph{worse} on
Hydrophobic Core ($\Delta=+0.73$, $p_{\mathrm{Holm}}<0.001$), the condition with the
weakest rank transfer, and therefore the one most in need of rescue.

Repeating the
analysis after the post-hoc degeneracy audit removes the Amylase win
($p_{\mathrm{Holm}}=0.108$) and leaves the Hydrophobic Core harm intact
($p_{\mathrm{Holm}}=0.002$). After the audit, augmentation still helps significantly only on Lipophilicity mismatch ($p_{\mathrm{Holm}}=0.038$), while Caco2 scaffold and Amylase move to $0.094$. No proxy-guided method significantly beats
validation-driven TPE in the pre-registered sequential family, while one
curvature-guided arm is significantly worse. These search results concern multi-objective search combined with a rank-sum deploy rule. Full paired results, filtered and
unfiltered, are in Appendix~\ref{app:pool}.

Thus, correlation with generalisation gap is not sufficient for model selection:
a generalisation signal need not be a model-selection signal.

\section{Failure analysis}
\label{sec:failure}

\paragraph{The proxy carries generalisation signal, but not standard curvature.}
After controlling for architecture and learning rate, the proxy retains a consistent
positive association with generalisation gap on molecular tasks
($\rho\approx0.16$--$0.35$), but not on Hydrophobic Core ($-0.30$) or GDSC2
($0.01$), while $\lambda_{\max}$ and $\operatorname{tr}(H)$ are weak or negative
(Appendix~\ref{app:geom}). Its marginal anti-correlation with the full-network
$\lambda_{\max}$ ($-0.40$ on Lipo mismatch, $-0.55$ on Hydrophobic Core) shrinks to
$-0.12$ and $-0.14$ under the same controls \citep{kaur2023hessian,cohen2021}. On Lipo
mismatch it anti-correlates with Hessian trace ($-0.61$) but correlates with relative
flatness ($+0.72$) \citep{petzka2021relative}, so it captures information
these scalar Hessian summaries miss.

\paragraph{The preferred model changes under extrapolation.}
On Hydrophobic Core, the deployment oracle uses a $2.05\times$ larger learning rate,
$4\times$ greater width and less regularisation than the validation-selected model{,
consistently across replications (Appendix~\ref{app:hpmech})}.
The shift changes which optimisation/capacity regime is preferable, not merely how
hard the examples are, consistent with training on below-median variants that
truncate the observed target range while deployment requires extrapolation to
higher-fitness variants. {Even the fittest training variant is less fit than every
deployment variant (Appendix~\ref{app:data}), so validation, drawn from the training
range, cannot reflect this extrapolation (Figure~\ref{fig:concept}a).}

\paragraph{Curvature can become null or inverted.}
Curvature shows little relationship with generalisation on the molecular tasks, and
Hydrophobic Core reverses the expected trend:
$\rho(\operatorname{tr}(H),L_{\mathrm{dep}}-L_{\mathrm{train}})\approx-0.52$, so
greater curvature accompanies better extrapolation, and $\lambda_{\max}$ is the
weakest shared-pool selector. {On Hydrophobic Core, with the architecture fixed, sharper minima deploy better
($\rho=-0.76$; Lipophilicity mismatch shows no relationship, $\rho=-0.05$;
Appendix~\ref{app:geom}, Figure~\ref{fig:concept}c). The proxy does not improve on
validation there either, although better candidates exist.}
 As deployment loss is training loss
$+$ generalisation gap, even accurate gap signals can misrank candidates whose training
losses differ {(Figure~\ref{fig:concept}b)}, so correlation with the gap is
insufficient for selection.

\paragraph{Geometry can reward collapse.}
On Amylase, the proxy improves median deployment MSE by $1.39$ and wins all 20
replications, yet its median pick has a dead-unit fraction of $0.998$ and near-zero
prediction variance. It matches the training-mean MSE ($1.145$), with Spearman
undefined in 14 of 20 replications: lower error without useful variant ranking
{(Figure~\ref{fig:concept}d)}. Collapse drives the proxy toward zero and $\lambda_{\max}$ to its minimum of 2, set by
the output bias (Appendix~\ref{app:protocol}), so low geometric scores can signal a
network that has learned little. {Under this
position shift such a network is close to optimal in MSE: in 18 of 20 replications no
audited candidate improves on predicting the training mean, and on the unfiltered pools
collapse even reverses rank transfer ($-0.75$), because collapsed predictors rank poorly
on validation but well on deployment.} The audit
removes this win through its dead-unit criterion; removing only constant predictors
does not (Appendix~\ref{app:degen}). {All models were trained for a fixed 100 epochs, so every selector ranks the same pool; although not every candidate had converged, the pools contain competitive models (Appendix~\ref{app:sanity}).}

\section{Conclusion}

Training-set measures can carry generalisation signal without reliably selecting the
best deployment model. Hessian summaries correlate weakly or negatively with
generalisation gap, while our proxy correlates positively on most tasks; yet across
biological shifts the proxy can help, fail, harm selection, or favour degenerate
predictors. This shows no consistent selection gain, not that such measures carry no
information. Until a measure beats validation at choosing the deployed model, it is
best treated as an auxiliary diagnostic: studies should report rank transfer, compare
against validation and simple baselines, and audit selections for collapse. Our
conclusions are limited to the models, representations, selection rules, and
geometric criteria studied here {(Appendix~\ref{app:limits})}.

\section*{Use of large language models}
Large language models were used to assist with code development, literature search,
and language editing. All experimental results, statistical analyses, and reported
numerical values were generated and verified by the authors. The authors reviewed
and take responsibility for all content in the manuscript.

\section*{Reproducibility statement}
Each outer run derives six seeds deterministically from its index, governing the data
split, the 48-candidate Latin hypercube, initialisation, batch ordering, the stochastic
proxy estimators, and the HPO sampler. Every trained candidate is written atomically as
a single row, and test predictions are retained for all candidates, so every post-hoc
diagnostic reuses the original trained models rather than retraining after observing an
outcome. Full protocol, search ranges, and per-condition results are in the appendix. Code is available at \url{https://github.com/AdityaNagarsekar/A-Generalisation-Signal-Need-Not-Be-a-Model-Selection-Signal}.

\begin{ack}
Snehanshu Saha acknowledges the DBT-Builder project (BT/INF/22/SP42543/2021),
Govt.\ of India, for partial support. The authors declare no competing interests.
\end{ack}

\bibliographystyle{plainnat}
\bibliography{references}








\appendix

\section{Limitations and scope}
\label{app:limits}

Our conclusions concern the tested MLP architectures, fixed molecular and sequence
representations, and the geometric criteria evaluated here. They do not establish
that all curvature- or representation-based selectors fail under biological
distribution shift, nor that the same behaviour will hold for pretrained biological
foundation models or learned representations.

The negative results also concern the selection rules we tested: deploying the lowest
score, deploying the lowest rank sum of validation loss and a signal, and multi-objective
search followed by that rank-sum rule. Deployment loss is training loss plus
generalisation gap, so a signal that tracks the gap does not by itself estimate
deployment loss, and it can misrank candidates whose training losses differ even when
its estimate of the gap is accurate. {Apart from the descriptive
validation-free comparison in Appendix~\ref{app:abl}, we did not test rules that combine a
geometric signal with training loss,} or that let the direction of the signal vary by condition.
Where the proxy's association with the gap is negative, as on Hydrophobic Core
(Table~\ref{tab:app-geom}), a rule that deploys the lowest score is misdirected by
construction. Our results show no consistent improvement in selection; they do not show
that these measures carry no information about deployment (Appendix~\ref{app:abl}).

The experiments also focus on supervised regression settings in which candidate
models are selected from controlled hyperparameter sweeps. Other deployment settings,
including uncertainty-aware selection, ensembling, active learning, or adaptation
after shift, may behave differently.

A natural next question is whether shift-aware or representation-aware signals can
identify the new deployment ordering when validation rank transfer collapses, rather
than relying on generic training-set geometry.

\section{Protocol, reproducibility, and deviations from the pre-registered plan}
\label{app:protocol}

\paragraph{Architecture.} A plain feed-forward MLP: \texttt{depth} hidden blocks of
\texttt{width} units, each Linear $\rightarrow$ ReLU $\rightarrow$ Dropout, then a
linear head to a scalar target. \textbf{No normalisation layers are used.} This matters
for interpretation: without BatchNorm to re-centre pre-activations, permanently
inactive ReLU units are considerably more common, which is the direct cause of the
degeneracy audited in \S\ref{app:degen}. It also means our setting is not the one in
which \citet{kaur2023hessian} observe BatchNorm improving generalisation without
reducing $\lambda_{\max}$.

\paragraph{Optimisation.} Adam with \emph{coupled} $L_2$ weight decay (the
\texttt{weight\_decay} argument, not AdamW's decoupled form), batch size 64, MSE on
standardised targets, 100 fixed epochs. No early stopping and no learning-rate
schedule. Early stopping is deliberately excluded: it is itself a validation-based
decision and would give validation a second selection opportunity the training-only
signals do not get.

\paragraph{Search space.} Each outer replication draws $K=48$ configurations from a
scrambled Latin hypercube over learning rate $\eta\in[10^{-4},3\times10^{-3}]$
(log-uniform), weight decay (exactly $20\%$ of candidates zero, remainder log-uniform
on $[10^{-6},10^{-2}]$), dropout $\in[0,0.5]$, width $\in\{64,128,256,512\}$ and depth
$\in\{1,2,3,4\}$. Width and depth are drawn stratified, so each level appears a fixed
number of times per pool.

\paragraph{Signals and curvature.} All signals are computed after training, with
dropout disabled, on a fixed subset of the training set: $\min(1024,n_{\text{train}})$
examples rounded down to a multiple of 64, drawn once per replication with the proxy
seed and split without shuffling into consecutive batches of 64. The proxy takes the
maximum over these batches (Eq.~\eqref{eq:final-proxy-2}). The dead-unit fraction used in
the audit is the share of hidden units, averaged over layers, that output zero on every
example of this subset. $\lambda_{\max}$ and the Hessian trace are computed on the same
subset for the training MSE alone, without the $L_2$ term, with respect to every
trainable parameter: the weights and biases of all layers, including the output bias.
$\lambda_{\max}$ uses power iteration with exact Hessian-vector products (up to 30
iterations, tolerance $10^{-3}$); the trace uses Hutchinson's estimator with 30
Rademacher probes. Because the output bias enters every prediction additively, its
diagonal Hessian entry under mean squared error is exactly 2, so $\lambda_{\max}\ge2$ for
every network; collapsed networks sit at this floor rather than at zero
(Appendix~\ref{app:degen}).

\paragraph{Selection rules and ties.} Every single-signal selector deploys the candidate
with the lowest score. Val+X ranks the pool on validation MSE and on signal X, using
average ranks for ties, and deploys the candidate with the lowest rank sum; the
multi-objective search arms apply the same rule to all observed trials
(Appendix~\ref{app:seq}). Any remaining tie goes to the lowest candidate index, which is
the Latin-hypercube draw order and carries no information about the hyperparameters.
Ties matter mainly for \#params: the smallest architecture in a pool appears between one
and five times, so \#params deploys the first-drawn of these networks, whose learning
rate, dropout and weight decay are effectively random.

\paragraph{Names.} Table~\ref{tab:app-names} maps the selector names used throughout
the paper to the identifiers in the code release. Figures label Lipophilicity (mismatch)
as \emph{Cond7} and Hydrophobic Core as \emph{Hydro}.

\begin{table}[htbp]\centering\scriptsize\setlength{\tabcolsep}{4pt}
\caption{Selector names, the model each deploys, and the identifier in the code release.}
\label{tab:app-names}
\begin{tabular}{lll}\toprule
name & deploys the candidate with the lowest & code \\\midrule
Val & validation MSE & \texttt{val} \\
Proxy & layer-averaged activation proxy, Eq.~\eqref{eq:final-proxy-2} & \texttt{fg\_legacy} \\
Proxy (pen.) & penultimate-layer proxy, Eq.~\eqref{eq:final-proxy-1} & \texttt{fg\_penult} \\
$\lambda_{\max}$ & top Hessian eigenvalue of the training MSE & \texttt{hess\_top} \\
Val+X & rank sum of validation MSE and signal X & \texttt{val+legacy}, \texttt{val+penult}, \texttt{val+hess} \\
Train & training MSE & \texttt{train\_mse} \\
\#params & parameter count & \texttt{n\_params} \\
Random & none; expected deployment loss of a uniform pick & \texttt{random(E)} \\
Oracle & deployment MSE; a bound, not a selector & \texttt{oracle} \\
\bottomrule\end{tabular}\end{table}

\paragraph{What an outer replication is.} Not a re-initialisation. Replication $r$
deterministically fixes six seeds as $10000r+1,\dots,10000r+6$, governing respectively
the train/validation/test split, \textbf{the 48-candidate hypercube itself}, weight
initialisation (offset by candidate index), batch composition and ordering, the
stochastic proxy estimators, and the HPO sampler. Two consequences matter for reading
the reported spreads. The candidate pool is redrawn every replication, so ``every
selector ranks the same 48 models'' holds \emph{within} a replication, which is what
makes each comparison paired, but not across them, and no result is a property of one
hypercube draw. The test set is redrawn as well on every condition \emph{except} the
two FLIP2 splits, whose partitions are deterministic; FLIP2 spreads therefore reflect
only training stochasticity and the pool redraw.

\paragraph{Replication counts.} 20 for both Caco2 conditions, Lipophilicity mismatch,
Amylase and GDSC2; 30 for Lipophilicity random and scaffold and Hydrophobic Core; 10
per arm for the sequential searches.

\paragraph{Protocol lock.} The conditions, replication counts, search space, training
settings, proxy subset and signals were frozen in a protocol file before the first
confirmatory run. It names the layer-averaged proxy as the primary signal and the
penultimate-layer proxy and $\lambda_{\max}$ as secondary signals, and it pre-registers
the sequential comparison of proxy-guided TPE against validation-only TPE on
Lipophilicity mismatch and Hydrophobic Core. Its SHA-256 hash is recorded on every
result row in the code release.

\paragraph{Statistical tests.} Every comparison is paired within replication. We use the
two-sided Wilcoxon signed-rank test (SciPy 1.15.3). Zero paired differences, common when
a rank-sum selector deploys the same model as validation, are discarded, following
Wilcoxon's original method. SciPy then chooses the null distribution from the data:
exact when no zero or tied absolute difference is present; otherwise a deterministic
permutation test for at most 13 pairs, as in the sequential searches, and the normal
approximation without continuity correction for the 20 or 30 pairs of the shared-pool
analyses. Recomputing with zeros removed beforehand and the exact null distribution
changes no significance decision in Table~\ref{tab:app-prereg}.
$p$-values below $0.001$ are reported as $<0.001$. Holm correction is applied within
four separate families:
\begin{enumerate}\setlength{\itemsep}{0pt}
\item \emph{Confirmatory} (Table~\ref{tab:app-prereg}): Proxy against Val across the four
confirmatory conditions, and separately Val+Proxy against Val across the same four. The
unfiltered (primary) and audited analyses are corrected separately.
\item \emph{Per-condition} (Tables~\ref{tab:app-pool-caco2-random}--\ref{tab:app-pool-gdsc-drug}):
the ten contrasts against Val within each condition, including Oracle and Random.
These are descriptive.
\item \emph{Metric robustness} (Table~\ref{tab:app-metric}): the eight selector contrasts
against Val within each condition and metric, excluding Oracle and Random.
\item \emph{Sequential search} (Table~\ref{tab:app-seq}): the eight arms against
validation-only TPE within each condition.
\end{enumerate}
The NDCG $p$-values quoted in \S\ref{sec:outcome} are unadjusted and use the audited
pools.

\paragraph{Deviations from the pre-registered plan.} Stated explicitly because they
affect how the results should be read.
\begin{enumerate}\setlength{\itemsep}{0pt}
\item \textbf{The degeneracy audit is post-hoc.} The plan pre-specified
\emph{recording} dead-unit and prediction-spread diagnostics and excluding trials that
fail training outright, but no exclusion rule. After diagnosing the Amylase result we
excluded a candidate if (i) the standard deviation of its test-set predictions is below
$10^{-6}$, making it a constant predictor, or (ii) more than half of its hidden ReLU
units output zero on every example of the training proxy subset. A replication would be
dropped if fewer than five candidates remained; none was. The threshold of one half
marks networks that have lost most of their capacity; it was set once, after seeing the
Amylase result, and not tuned. Criterion (ii) is the one that matters: it also removes
partially inactive networks, and on its own it reproduces the audited Amylase result,
whereas criterion (i) on its own leaves the proxy's advantage intact
(Table~\ref{tab:app-collapse-filters}). Criterion (i) reads deployment predictions, so
the audit is a diagnostic, not a rule a practitioner could apply before deployment.
Evaluating (i) on training-set predictions instead changes the outcome for 4 of the
59{,}712 stored candidates, all on Amylase, and leaves Table~\ref{tab:app-prereg}
unchanged. Primary inference in \S\ref{sec:outcome} therefore uses the unfiltered pools;
the audited analysis is reported alongside as a diagnostic, and each appendix table and
figure states which pools it uses.
\item \textbf{The confirmatory family is four conditions.} The pre-registered Holm
family is $\{$Lipophilicity mismatch, Caco2 scaffold, Amylase, Hydrophobic Core$\}$;
Lipophilicity \emph{scaffold} was designated descriptive on the grounds that its
validation split is itself scaffold-OOD. Caco2 scaffold is built by the same TDC
scaffold split and shares that property, so the family's composition is a design choice
rather than a property of the splits. Adding Lipophilicity scaffold as a fifth member
leaves every Proxy decision unchanged but weakens Val+Proxy: on the unfiltered pools it
remains significant only on Amylase ($p_{\mathrm{Holm}}=0.036$) and, as a harm, on
Hydrophobic Core ($<0.001$), with Lipophilicity mismatch at $0.069$ and Caco2 scaffold at
$0.094$; on the audited pools only the Hydrophobic Core harm remains ($0.002$). The
designation also bears on our one positive search result, which sits on Lipophilicity
scaffold and is additionally post-hoc within the sequential family.
\item \textbf{Analyses added after unblinding.} GDSC2, the shift-severity sweep, the
penultimate-layer search arms, sequential search on the six conditions beyond
Lipophilicity mismatch and Hydrophobic Core, and the Hessian-trace diagnostics. The
analyses added in response to review (Tables~\ref{tab:app-collapse}
and~\ref{tab:app-collapse-filters}, the intervals in Table~\ref{tab:app-prereg}, and the
sensitivity analysis in item 2) reuse the stored models and are reproduced by
\texttt{scripts\_icbinb/review\_diagnostics.py}. All are labelled exploratory and never
pooled with confirmatory tests.
\end{enumerate}

\section{Datasets and split construction}
\label{app:data}

\paragraph{Molecular.} Caco2-Wang permeability and Lipophilicity from TDC
\citep{huang2021tdc}, as 2048-bit radius-2 Morgan fingerprints
\citep{rogers2010ecfp}. Random and Bemis--Murcko scaffold splits at 70/10/20 give
matched interpolation and structural-shift conditions on identical data and
featurisation.

\paragraph{The mismatch condition.} A standard scaffold split does not create the
failure mode of interest, because validation and test both contain unseen scaffolds -
validation is already OOD and can still rank OOD candidates. We therefore keep the
scaffold-disjoint test set unchanged and randomly repartition only the original
train--validation pool. Validation becomes in-distribution with respect to training
while deployment remains scaffold-disjoint. This is the condition the study was
designed around and the one where validation \emph{ought} to fail.

\paragraph{Protein fitness.} FLIP2 Amylase (close-to-far, a position split) and
Hydrophobic Core (low-to-high, a fitness split), benchmark-provided and unmodified
\citep{dallago2021flip,didi2026flip2}, one-hot encoded and padded to fixed length. Our
Hydrophobic Core split reproduces the published one exactly (24{,}935 variants; median
boundary $-3.206$ against the stated $-3.21$) and Amylase at 3{,}706 variants. The
target ranges do not overlap: the highest training fitness is below the lowest test
fitness, so every deployment variant is more functional than anything seen in training.

\paragraph{Drug response.} GDSC2 with whole compounds held out
\citep{yang2013genomics,iorio2016landscape}: validation measurements involve compounds
also present in training, while test compounds are chemically novel. Features
concatenate a drug fingerprint with cell-line expression, with variance filtering and
standardisation fitted on training data only.

\section{Derivation and bound}
\label{app:proxy}

\subsection{Setup and notation}

Let $f_\theta$ be an $L$-layer feedforward network with a linear output layer,
trained with mean squared error. The hidden nonlinearity plays no role in what
follows, only that the output layer is affine. Let $m$ denote the number of
training examples, $a_i:=a^{[L-1]}(x_i)\in\mathbb{R}^d$ the penultimate
activation for training example $i$, and
$W^{[L]}\in\mathbb{R}^{k\times d}$ the final affine layer. The loss is
\[
E(W^{[L]})
=
\frac{1}{2m}\sum_{i=1}^m
\left\|W^{[L]}a_i-y_i\right\|_2^2.
\]

For a mini-batch $\mathcal B$, let
$A_{\mathcal B}\in\mathbb{R}^{|\mathcal B|\times d}$ denote the matrix whose
rows are the corresponding penultimate activations.

\subsection{Statement}

\begin{proposition}
\label{prop:bound}

The Hessian with respect to the final-layer weights is
\[
\nabla^2_{W^{[L]}}E
=
\frac{1}{m}\sum_{i=1}^m
I_k\otimes(a_i a_i^\top),
\]
and its Frobenius norm satisfies
\begin{equation}
\label{eq:hessian-frob-global}
\frac{1}{m}\max_{1\le i\le m}\|a_i\|_2^2
\le
\left\|\nabla^2_{W^{[L]}}E\right\|_F
\le
\frac{\sqrt{k}}{m}
\sum_{i=1}^m\|a_i\|_2^2.
\end{equation}

If the training set is partitioned into mini-batches
$\mathcal B_1,\ldots,\mathcal B_T$, define
\[
\mu_t^2
:=
\frac{1}{|\mathcal B_t|}
\sum_{i\in\mathcal B_t}\|a_i\|_2^2
=
\frac{\|A_{\mathcal B_t}\|_F^2}{|\mathcal B_t|}.
\]
Then
\begin{equation}
\label{eq:hessian-frob-batch}
\frac{1}{m}\max_{1\le t\le T}\mu_t^2
\le
\left\|\nabla^2_{W^{[L]}}E\right\|_F
\le
\sqrt{k}\max_{1\le t\le T}\mu_t^2.
\end{equation}

\end{proposition}

Because $a_i$ does not involve $W^{[L]}$, the loss is an exact quadratic in
$W^{[L]}$ and the Hessian above is constant in $W^{[L]}$, depending on the
network only through its activations.

\subsection{Proof}

\begin{proof}

The Hessian with respect to the final-layer weights is
\[
\nabla^2_{W^{[L]}}E
=
\frac{1}{m}\sum_{i=1}^m
I_k\otimes(a_i a_i^\top).
\]

Using
\[
\|I_k\otimes M\|_F
=
\sqrt{k}\,\|M\|_F,
\qquad
\|a_i a_i^\top\|_F
=
\|a_i\|_2^2,
\]
and the triangle inequality,
\[
\left\|\nabla^2_{W^{[L]}}E\right\|_F
\le
\frac{\sqrt{k}}{m}
\sum_{i=1}^m\|a_i\|_2^2.
\]

Since $\|M\|_F\ge\|M\|_2$ for any matrix,
\[\|\nabla^2_{W^{[L]}}E\|_F\ge\|\nabla^2_{W^{[L]}}E\|_2 .\]
Each summand $I_k\otimes(a_ia_i^\top)$ is positive semidefinite, so
$\nabla^2_{W^{[L]}}E\succeq\tfrac1m I_k\otimes(a_ja_j^\top)$ for every $j$, hence
\[\|\nabla^2_{W^{[L]}}E\|_2=\lambda_{\max}\!\big(\nabla^2_{W^{[L]}}E\big)\ge\frac1m\max_i\|a_i\|_2^2 .\]
From the spectral-norm bound,
\[
\left\|\nabla^2_{W^{[L]}}E\right\|_2
\ge
\frac{1}{m}\max_{1\le i\le m}\|a_i\|_2^2.
\]
Thus,
\[
\frac{1}{m}\max_{1\le i\le m}\|a_i\|_2^2
\le
\left\|\nabla^2_{W^{[L]}}E\right\|_F
\le
\frac{\sqrt{k}}{m}
\sum_{i=1}^m\|a_i\|_2^2,
\]
which proves Eq.~\eqref{eq:hessian-frob-global}.

Now partition the training set into mini-batches
$\mathcal B_1,\ldots,\mathcal B_T$ such that
\[
\bigcup_{t=1}^T\mathcal B_t=\{1,\ldots,m\},
\qquad
\mathcal B_s\cap\mathcal B_t=\varnothing
\quad(s\neq t).
\]
For each batch,
\[
\mu_t^2
=
\frac{1}{|\mathcal B_t|}
\sum_{i\in\mathcal B_t}\|a_i\|_2^2
=
\frac{\|A_{\mathcal B_t}\|_F^2}{|\mathcal B_t|}.
\]

For every $t$,
\[
\max_{i\in\mathcal B_t}\|a_i\|_2^2
\ge
\frac{1}{|\mathcal B_t|}
\sum_{i\in\mathcal B_t}\|a_i\|_2^2
=
\mu_t^2.
\]
Because the batches form a partition,
\[
\max_{1\le i\le m}\|a_i\|_2^2
=
\max_{1\le t\le T}
\max_{i\in\mathcal B_t}\|a_i\|_2^2.
\]
Therefore,
\[
\max_{1\le i\le m}\|a_i\|_2^2
\ge
\max_{1\le t\le T}\mu_t^2.
\]
Substituting into the lower bound in
Eq.~\eqref{eq:hessian-frob-global} gives
\[
\left\|\nabla^2_{W^{[L]}}E\right\|_F
\ge
\frac{1}{m}\max_{1\le t\le T}\mu_t^2.
\]

For the upper bound, decompose the sum over the same partition:
\[
\frac{1}{m}\sum_{i=1}^m\|a_i\|_2^2
=
\frac{1}{m}
\sum_{t=1}^T
|\mathcal B_t|\mu_t^2.
\]
Define
\[
w_t:=\frac{|\mathcal B_t|}{m},
\qquad
w_t\ge0,
\qquad
\sum_{t=1}^T w_t=1.
\]
Then
\[
\frac{1}{m}\sum_{i=1}^m\|a_i\|_2^2
=
\sum_{t=1}^T w_t\mu_t^2
\le
\max_{1\le t\le T}\mu_t^2.
\]
Substituting into the upper bound in
Eq.~\eqref{eq:hessian-frob-global} gives
\[
\left\|\nabla^2_{W^{[L]}}E\right\|_F
\le
\sqrt{k}\max_{1\le t\le T}\mu_t^2.
\]
Combining the two inequalities proves
Eq.~\eqref{eq:hessian-frob-batch}.

Finally, define
\[
R(\text{architecture},B)
:=
\max_{\mathcal B}
\frac{\|A_{\mathcal B}\|_F^2}{|\mathcal B|}.
\]
Since
\[
R(\text{architecture},B)
=
\max_{1\le t\le T}\mu_t^2,
\]
Eq.~\eqref{eq:hessian-frob-batch} gives
\begin{equation}
\label{eq:hessian-R-bound}
\frac{1}{m}R(\text{architecture},B)
\le
\left\|\nabla^2_{W^{[L]}}E\right\|_F
\le
\sqrt{k}\,R(\text{architecture},B).
\end{equation}

\end{proof}

\subsection{Why a valid bound need not give a usable ranking}
\label{app:proxy-caveat}

Proposition~\ref{prop:bound} bounds the Frobenius norm of the Hessian
pointwise using statistics of the penultimate activations. The bound is valid
for each model individually, but it does not imply that the induced quantity
preserves the ordering of different models. Hyperparameter search compares
models with different learning rates, depths, and widths, all of which can
change the activations. Consequently, a correct pointwise bound need not
produce a correct ranking across models.

Section~\ref{sec:failure} is consistent with this caveat but does not test it
directly, because the quantities it compares are not those of
Proposition~\ref{prop:bound}. The proposition concerns the Hessian with respect to the
final-layer weights and the unnormalised penultimate activation energy.
Section~\ref{sec:failure} compares the width-normalised, layer-averaged proxy with
$\lambda_{\max}$ and $\operatorname{tr}(H)$ of the full network, taken over all weights
and biases (Appendix~\ref{app:protocol}). The anti-correlation reported there therefore
does not test whether the statistic in Proposition~\ref{prop:bound} preserves the
ordering of final-layer curvature across models; we did not measure the final-layer
Hessian separately.

\subsection{From the bound to the proxy used in experiments}
\label{app:proxy-steps}

Proposition~\ref{prop:bound} motivates the proxy but does not cover it. Two steps go
beyond the bound. First, to compare architectures of different hidden width $w$, we
divide $R(\text{architecture},B)$ by the width:
\begin{equation}
\label{eq:final-proxy-1}
P(h)
=
\max_\mathcal B
\frac{\lVert A_{\mathcal B}\rVert_F^2}{|\mathcal B|\,w}
\end{equation}

This width normalisation is not part of the proposition, and it can change the
ordering of models of different width. Second, we average the normalised quantity over
all hidden layers:

\begin{equation}
\label{eq:final-proxy-2}
P(h)
=
\max_\mathcal B \frac{1}{L}\sum_{\ell=1}^{L}
\frac{\lVert A_{\ell,\mathcal B}\rVert_F^2}{|\mathcal B|\,w_\ell}
\end{equation}
Layer averaging is likewise outside the proposition, which concerns only the
penultimate layer. Across candidates, the penultimate-layer form~\eqref{eq:final-proxy-1}
and the layer-averaged form~\eqref{eq:final-proxy-2} are strongly correlated
($\rho=+0.95$, Appendix~\ref{app:pool}); this addresses the averaging step but not width
normalisation. Form~\eqref{eq:final-proxy-2} was fixed as the primary signal, and
form~\eqref{eq:final-proxy-1} as a secondary signal, in the frozen protocol before any
confirmatory run (Appendix~\ref{app:protocol}); neither choice was made by examining
confirmatory results.

\section{Complete shared-pool results}
\label{app:pool}

Section~\ref{sec:outcome} reports a compressed summary; this section gives every
selector on every condition so a reader can verify that the summary hides nothing
inconvenient.

\paragraph{The pre-registered analysis.} Table~\ref{tab:app-prereg} is the confirmatory
result, reported twice: on the original unfiltered pools, which is the pre-registered
primary analysis, and again after the post-hoc degeneracy audit. The two differ
materially only where degeneracy is present. Unfiltered, replacing validation with the
proxy is Holm-significant on exactly one condition, Amylase, at $20/20$ replications,
and Appendix~\ref{app:degen} shows that the gain is lower error from near-constant
predictions rather than useful variant ranking. On the unfiltered pools, augmenting
validation is significant in both directions, improving three conditions and harming
Hydrophobic Core. After the audit, only the improvement on Lipophilicity mismatch
($p_{\mathrm{Holm}}=0.038$) and the harm on Hydrophobic Core ($0.002$) remain, while
Caco2 scaffold and Amylase move to $0.094$.

\paragraph{Ties and uncertainty.} Table~\ref{tab:app-prereg} reports, for each contrast,
the median paired difference, a 95\% percentile bootstrap interval for it (4000
resamples of replications), and wins, losses and ties. Rank-sum selectors often deploy
the same model as validation, so many differences are exactly zero; the median can then
be zero while the signed-rank test, which discards ties, is significant, as for
Val+Proxy on Caco2 scaffold with 7 wins, 3 losses and 10 ties. A change in the adjusted
$p$-value need not reflect a change in the estimated effect. The audit removes no Caco2
scaffold candidates, so that condition's paired differences, ties and interval are
identical in both analyses; only its Holm adjustment moves, because the other
conditions' $p$-values change.

\paragraph{Per-condition detail.} The tables that follow give, for every selector,
the median and interquartile range over replications, the paired median difference
against validation, win/loss counts, and raw and Holm-adjusted $p$-values within each
condition's family. These tables use the audited pools; the unfiltered confirmatory
contrasts are in Table~\ref{tab:app-prereg}. Penultimate-layer proxy results appear here rather than in the main
text; they track the layer-averaged variant closely ($\rho=+0.95$ between the two
signals), which is why only one is reported in the body.

\paragraph{Metric robustness.} Table~\ref{tab:app-metric} repeats the comparison on the
audited pools under MSE, MAE and Spearman. This is where the fragility of the one
shared-pool geometric win can be checked, and where the robustness of the \#params
baseline on Hydrophobic Core, which survives all three metrics, is visible. NDCG,
reported in Table~\ref{tab:all_results}, is not repeated there.

\paragraph{Metric definitions.} Spearman and NDCG compare the deployed model's test
predictions with the measured targets. NDCG uses the whole deployment set with no
cutoff, linear gains equal to each variant's fitness minus the lowest fitness in the
deployment set, and a $\log_2$ position discount, with the predictions as ranking scores
(scikit-learn \texttt{ndcg\_score}). Without a cutoff NDCG has a high floor: uniform
selection already scores $0.86$ on Amylase and $0.91$ on Hydrophobic Core, so
differences between selectors are small in absolute terms. Spearman is undefined when
the prediction vector is exactly constant; such pairs are dropped from the paired tests. No undefined values occur in the audited pools, so all replications
contribute to Table~\ref{tab:all_results}. On the unfiltered Amylase pools, the deployed
model's Spearman is undefined in 14 of 20 replications for Proxy and 16 of 20 for
$\lambda_{\max}$ (Table~\ref{tab:app-collapse}).

\begin{figure}[htbp]
\centering
\includegraphics[width=\textwidth]{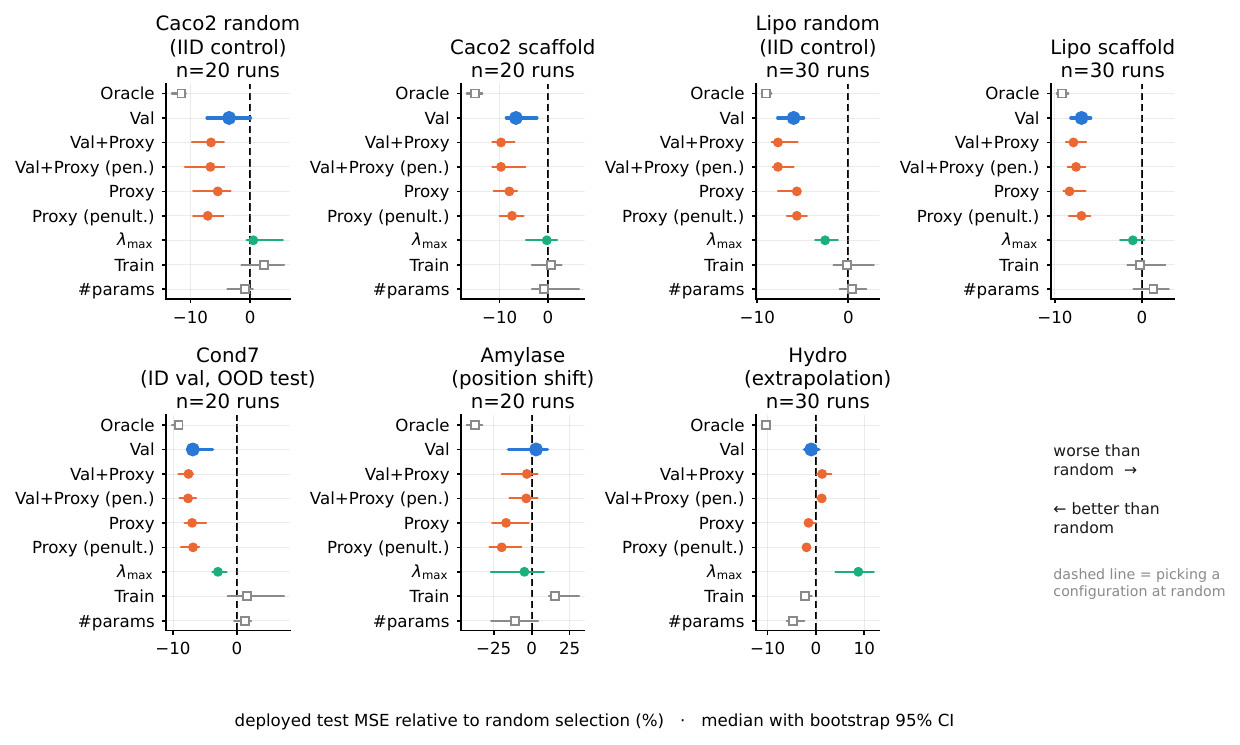}
\caption{Every selector against blind selection, per condition. Bars below the line
indicate a selector that deploys a better model than picking uniformly at random from
the same pool. A selector that cannot clear this line carries no usable
model-selection information, whatever its correlation with the generalisation gap.
Audited pools; Cond7 is Lipophilicity (mismatch) and Hydro is Hydrophobic Core.}
\label{fig:app-vsrandom}
\end{figure}

\begin{table}[htbp]\centering\scriptsize\setlength{\tabcolsep}{3pt}
\caption{Shared pool, Caco2 (random), audited pools (20 outer runs, median 48/48 candidates kept).}
\label{tab:app-pool-caco2-random}
\begin{tabular}{lrrrrr}\toprule
selector & median [IQR] & $\Delta$ vs Val & w/l & $p$ & $p_{\text{Holm}}$ \\\midrule
Oracle & 0.3354 [0.2945, 0.3642] & -0.0292 & 20/0 & $<$0.001 & $<$0.001 \\
Val & 0.3575 [0.3283, 0.3936] & -- & -- & -- & -- \\
Val+Proxy & 0.3446 [0.3144, 0.3749] & -0.0085 & 14/2 & 0.013 & 0.105 \\
Val+Proxy (pen.) & 0.3446 [0.3144, 0.3767] & -0.0105 & 13/2 & 0.009 & 0.081 \\
Val+$\lambda_{\max}$ & 0.3673 [0.3203, 0.3882] & -0.0113 & 12/5 & 0.177 & 0.355 \\
Proxy & 0.3490 [0.3178, 0.3723] & -0.0103 & 15/5 & 0.033 & 0.164 \\
Proxy (pen.) & 0.3449 [0.3158, 0.3723] & -0.0098 & 14/5 & 0.016 & 0.110 \\
$\lambda_{\max}$ & 0.3694 [0.3351, 0.4156] & +0.0114 & 7/13 & 0.076 & 0.233 \\
Train & 0.3851 [0.3536, 0.4047] & +0.0291 & 5/15 & 0.017 & 0.110 \\
\#params & 0.3597 [0.3363, 0.3991] & +0.0110 & 8/12 & 0.189 & 0.355 \\
Random & 0.3849 [0.3339, 0.4129] & +0.0128 & 6/14 & 0.058 & 0.233 \\
\bottomrule\end{tabular}\end{table}

\begin{table}[htbp]\centering\scriptsize\setlength{\tabcolsep}{3pt}
\caption{Shared pool, Caco2 (scaffold), audited pools (20 outer runs, median 48/48 candidates kept).}
\label{tab:app-pool-caco2-scaffold}
\begin{tabular}{lrrrrr}\toprule
selector & median [IQR] & $\Delta$ vs Val & w/l & $p$ & $p_{\text{Holm}}$ \\\midrule
Oracle & 0.4448 [0.4168, 0.4744] & -0.0376 & 18/0 & $<$0.001 & 0.002 \\
Val & 0.4750 [0.4462, 0.5269] & -- & -- & -- & -- \\
Val+Proxy & 0.4766 [0.4347, 0.5090] & +0.0000 & 7/3 & 0.047 & 0.375 \\
Val+Proxy (pen.) & 0.4766 [0.4347, 0.5179] & +0.0000 & 7/5 & 0.117 & 0.632 \\
Val+$\lambda_{\max}$ & 0.4782 [0.4549, 0.5277] & +0.0000 & 7/8 & 0.496 & 0.709 \\
Proxy & 0.4774 [0.4484, 0.5161] & -0.0189 & 11/7 & 0.157 & 0.632 \\
Proxy (pen.) & 0.4774 [0.4484, 0.5226] & +0.0026 & 9/10 & 0.355 & 0.709 \\
$\lambda_{\max}$ & 0.5017 [0.4627, 0.5310] & +0.0281 & 6/13 & 0.126 & 0.632 \\
Train & 0.5302 [0.4653, 0.5775] & +0.0323 & 6/14 & 0.105 & 0.632 \\
\#params & 0.5221 [0.4869, 0.5702] & +0.0283 & 4/16 & 0.076 & 0.531 \\
Random & 0.5188 [0.4938, 0.5691] & +0.0319 & 5/15 & 0.024 & 0.216 \\
\bottomrule\end{tabular}\end{table}

\begin{table}[htbp]\centering\scriptsize\setlength{\tabcolsep}{3pt}
\caption{Shared pool, Lipo (random), audited pools (30 outer runs, median 47/48 candidates kept).}
\label{tab:app-pool-lipo-random}
\begin{tabular}{lrrrrr}\toprule
selector & median [IQR] & $\Delta$ vs Val & w/l & $p$ & $p_{\text{Holm}}$ \\\midrule
Oracle & 0.3964 [0.3847, 0.4142] & -0.0110 & 23/0 & $<$0.001 & $<$0.001 \\
Val & 0.4089 [0.3963, 0.4259] & -- & -- & -- & -- \\
Val+Proxy & 0.4029 [0.3867, 0.4239] & +0.0000 & 11/2 & 0.019 & 0.074 \\
Val+Proxy (pen.) & 0.4044 [0.3867, 0.4241] & +0.0000 & 11/3 & 0.019 & 0.074 \\
Val+$\lambda_{\max}$ & 0.4188 [0.4009, 0.4381] & +0.0084 & 6/20 & 0.001 & 0.005 \\
Proxy & 0.4113 [0.3927, 0.4313] & -0.0004 & 15/9 & 0.775 & 0.775 \\
Proxy (pen.) & 0.4143 [0.3985, 0.4341] & +0.0001 & 12/15 & 0.230 & 0.459 \\
$\lambda_{\max}$ & 0.4245 [0.4063, 0.4471] & +0.0164 & 4/26 & $<$0.001 & $<$0.001 \\
Train & 0.4417 [0.4270, 0.4598] & +0.0268 & 3/27 & $<$0.001 & $<$0.001 \\
\#params & 0.4405 [0.4247, 0.4634] & +0.0311 & 2/28 & $<$0.001 & $<$0.001 \\
Random & 0.4345 [0.4201, 0.4621] & +0.0271 & 1/29 & $<$0.001 & $<$0.001 \\
\bottomrule\end{tabular}\end{table}

\begin{table}[htbp]\centering\scriptsize\setlength{\tabcolsep}{3pt}
\caption{Shared pool, Lipo (scaffold), audited pools (30 outer runs, median 47/48 candidates kept).}
\label{tab:app-pool-lipo-scaffold}
\begin{tabular}{lrrrrr}\toprule
selector & median [IQR] & $\Delta$ vs Val & w/l & $p$ & $p_{\text{Holm}}$ \\\midrule
Oracle & 0.4942 [0.4573, 0.5307] & -0.0114 & 20/0 & $<$0.001 & 0.001 \\
Val & 0.5129 [0.4716, 0.5505] & -- & -- & -- & -- \\
Val+Proxy & 0.5001 [0.4743, 0.5491] & +0.0000 & 12/5 & 0.093 & 0.371 \\
Val+Proxy (pen.) & 0.5001 [0.4759, 0.5505] & +0.0000 & 12/6 & 0.102 & 0.371 \\
Val+$\lambda_{\max}$ & 0.5255 [0.4825, 0.5535] & +0.0123 & 5/20 & 0.003 & 0.013 \\
Proxy & 0.5113 [0.4654, 0.5373] & +0.0000 & 14/10 & 0.331 & 0.663 \\
Proxy (pen.) & 0.5109 [0.4629, 0.5560] & +0.0005 & 12/15 & 0.564 & 0.663 \\
$\lambda_{\max}$ & 0.5405 [0.5026, 0.5832] & +0.0357 & 2/28 & $<$0.001 & $<$0.001 \\
Train & 0.5462 [0.5165, 0.6021] & +0.0479 & 1/29 & $<$0.001 & $<$0.001 \\
\#params & 0.5464 [0.5123, 0.5907] & +0.0440 & 3/26 & $<$0.001 & $<$0.001 \\
Random & 0.5460 [0.5130, 0.5866] & +0.0387 & 1/29 & $<$0.001 & $<$0.001 \\
\bottomrule\end{tabular}\end{table}

\begin{table}[htbp]\centering\scriptsize\setlength{\tabcolsep}{3pt}
\caption{Shared pool, Lipo (mismatch), audited pools (20 outer runs, median 47/48 candidates kept).}
\label{tab:app-pool-cond7}
\begin{tabular}{lrrrrr}\toprule
selector & median [IQR] & $\Delta$ vs Val & w/l & $p$ & $p_{\text{Holm}}$ \\\midrule
Oracle & 0.4768 [0.4621, 0.5161] & -0.0089 & 19/0 & $<$0.001 & 0.001 \\
Val & 0.4947 [0.4770, 0.5319] & -- & -- & -- & -- \\
Val+Proxy & 0.4775 [0.4721, 0.5304] & -0.0004 & 10/1 & 0.013 & 0.043 \\
Val+Proxy (pen.) & 0.4786 [0.4721, 0.5304] & -0.0014 & 11/2 & 0.011 & 0.043 \\
Val+$\lambda_{\max}$ & 0.5159 [0.4853, 0.5509] & +0.0169 & 4/15 & 0.003 & 0.015 \\
Proxy & 0.5058 [0.4733, 0.5314] & -0.0016 & 13/4 & 0.309 & 0.618 \\
Proxy (pen.) & 0.5004 [0.4759, 0.5341] & -0.0021 & 13/7 & 0.409 & 0.618 \\
$\lambda_{\max}$ & 0.5211 [0.4927, 0.5536] & +0.0189 & 3/17 & 0.001 & 0.005 \\
Train & 0.5450 [0.5274, 0.6028] & +0.0463 & 1/19 & $<$0.001 & $<$0.001 \\
\#params & 0.5388 [0.5058, 0.5827] & +0.0409 & 0/20 & $<$0.001 & $<$0.001 \\
Random & 0.5347 [0.5110, 0.5744] & +0.0380 & 0/20 & $<$0.001 & $<$0.001 \\
\bottomrule\end{tabular}\end{table}

\begin{table}[htbp]\centering\scriptsize\setlength{\tabcolsep}{3pt}
\caption{Shared pool, Amylase, audited pools (20 outer runs, median 11/48 candidates kept).}
\label{tab:app-pool-amylase}
\begin{tabular}{lrrrrr}\toprule
selector & median [IQR] & $\Delta$ vs Val & w/l & $p$ & $p_{\text{Holm}}$ \\\midrule
Oracle & 1.4681 [1.3505, 1.6811] & -0.8348 & 19/0 & $<$0.001 & 0.001 \\
Val & 2.5773 [2.1102, 2.6980] & -- & -- & -- & -- \\
Val+Proxy & 2.3381 [1.8289, 2.5956] & +0.0000 & 8/2 & 0.093 & 0.556 \\
Val+Proxy (pen.) & 2.3149 [2.0077, 2.5956] & +0.0000 & 8/3 & 0.131 & 0.653 \\
Val+$\lambda_{\max}$ & 2.6598 [2.5062, 2.8849] & +0.0000 & 3/6 & 0.214 & 0.737 \\
Proxy & 2.1202 [1.7053, 2.4845] & -0.3628 & 14/5 & 0.027 & 0.188 \\
Proxy (pen.) & 1.9605 [1.6724, 2.2805] & -0.4187 & 14/5 & 0.020 & 0.157 \\
$\lambda_{\max}$ & 2.3807 [1.6241, 2.7906] & +0.0000 & 9/9 & 0.420 & 0.841 \\
Train & 2.8910 [2.7190, 3.2768] & +0.3046 & 1/14 & 0.001 & 0.007 \\
\#params & 2.2475 [1.6394, 2.5168] & -0.1447 & 11/8 & 0.184 & 0.737 \\
Random & 2.5144 [2.3494, 2.5485] & -0.0735 & 12/8 & 0.784 & 0.841 \\
\bottomrule\end{tabular}\end{table}

\begin{table}[htbp]\centering\scriptsize\setlength{\tabcolsep}{3pt}
\caption{Shared pool, Hydrophobic Core, audited pools (30 outer runs, median 34/48 candidates kept).}
\label{tab:app-pool-hydro}
\begin{tabular}{lrrrrr}\toprule
selector & median [IQR] & $\Delta$ vs Val & w/l & $p$ & $p_{\text{Holm}}$ \\\midrule
Oracle & 20.5821 [20.2789, 20.9410] & -2.2398 & 30/0 & $<$0.001 & $<$0.001 \\
Val & 22.9161 [22.3655, 23.3435] & -- & -- & -- & -- \\
Val+Proxy & 23.4878 [23.0254, 23.8202] & +0.5885 & 5/20 & $<$0.001 & 0.004 \\
Val+Proxy (pen.) & 23.2913 [23.0405, 23.7761] & +0.6543 & 6/20 & 0.003 & 0.020 \\
Val+$\lambda_{\max}$ & 23.4338 [22.4991, 24.2371] & +0.2414 & 11/19 & 0.114 & 0.343 \\
Proxy & 22.8071 [22.4965, 23.0753] & -0.0748 & 16/14 & 0.792 & 0.792 \\
Proxy (pen.) & 22.6362 [22.4050, 22.8820] & -0.1671 & 21/8 & 0.071 & 0.284 \\
$\lambda_{\max}$ & 25.3299 [23.5878, 26.1371] & +2.4095 & 3/27 & $<$0.001 & $<$0.001 \\
Train & 22.4987 [22.3354, 22.9085] & -0.3221 & 20/10 & 0.047 & 0.236 \\
\#params & 22.0816 [21.6544, 22.5313] & -0.2283 & 18/2 & 0.001 & 0.008 \\
Random & 23.0458 [22.9525, 23.1819] & +0.2272 & 11/19 & 0.152 & 0.343 \\
\bottomrule\end{tabular}\end{table}

\begin{table}[htbp]\centering\scriptsize\setlength{\tabcolsep}{3pt}
\caption{Shared pool, GDSC2, audited pools (20 outer runs, median 48/48 candidates kept).}
\label{tab:app-pool-gdsc-drug}
\begin{tabular}{lrrrrr}\toprule
selector & median [IQR] & $\Delta$ vs Val & w/l & $p$ & $p_{\text{Holm}}$ \\\midrule
Oracle & 0.6666 [0.4961, 0.7516] & -0.0529 & 19/0 & $<$0.001 & 0.001 \\
Val & 0.7086 [0.5330, 0.8320] & -- & -- & -- & -- \\
Val+Proxy & 0.7116 [0.5428, 0.8201] & -0.0018 & 10/8 & 0.983 & 1.000 \\
Val+Proxy (pen.) & 0.7181 [0.5428, 0.8201] & +0.0000 & 9/9 & 0.647 & 1.000 \\
Val+$\lambda_{\max}$ & 0.6925 [0.5297, 0.8018] & +0.0000 & 9/7 & 0.255 & 1.000 \\
Proxy & 0.7256 [0.5899, 0.8203] & +0.0286 & 8/12 & 0.177 & 1.000 \\
Proxy (pen.) & 0.7347 [0.5507, 0.8448] & +0.0242 & 7/13 & 0.090 & 0.807 \\
$\lambda_{\max}$ & 0.7047 [0.5481, 0.8092] & +0.0003 & 10/10 & 0.622 & 1.000 \\
Train & 0.7335 [0.5683, 0.8009] & +0.0131 & 7/12 & 0.334 & 1.000 \\
\#params & 0.6868 [0.5423, 0.8294] & +0.0134 & 9/11 & 0.388 & 1.000 \\
Random & 0.7295 [0.5631, 0.8023] & +0.0059 & 9/11 & 0.452 & 1.000 \\
\bottomrule\end{tabular}\end{table}

\begin{table}[htbp]\centering\scriptsize\setlength{\tabcolsep}{3pt}
\caption{The pre-registered confirmatory analysis, deployment MSE, Holm-corrected across the four-condition family within each selector. Reported on the original unfiltered pools (the pre-registered primary analysis) and again on the audited pools (post-hoc). $\Delta$ is the paired median difference against Val, negative favouring the selector; the interval is a 95\% percentile bootstrap over replications; w/l/t counts wins, losses and ties. The signed-rank test discards ties (Appendix~\ref{app:protocol}).}
\label{tab:app-prereg}
\begin{tabular}{lrrrrrrrr}\toprule
& \multicolumn{4}{c}{unfiltered (primary)} & \multicolumn{4}{c}{audited} \\
\cmidrule(lr){2-5}\cmidrule(lr){6-9}
condition & $\Delta$ & 95\% CI & w/l/t & $p_{\text{Holm}}$ & $\Delta$ & 95\% CI & w/l/t & $p_{\text{Holm}}$ \\\midrule
\multicolumn{9}{l}{\emph{Proxy against Val}} \\
Lipo (mismatch) & -0.0009 & [-0.0068, 0.0000] & 12/5/3 & 0.407 & -0.0016 & [-0.0073, 0.0000] & 13/4/3 & 0.618 \\
Caco2 (scaffold) & -0.0189 & [-0.0301, +0.0093] & 11/7/2 & 0.328 & -0.0189 & [-0.0301, +0.0093] & 11/7/2 & 0.471 \\
Amylase & -1.3934 & [-1.5327, -1.1955] & 20/0/0 & \textbf{$<$0.001} & -0.3628 & [-0.6133, -0.0268] & 14/5/1 & 0.108 \\
Hydrophobic Core & +0.1634 & [-0.1005, +0.5858] & 12/18/0 & 0.328 & -0.0748 & [-0.4902, +0.3885] & 16/14/0 & 0.792 \\
\midrule
\multicolumn{9}{l}{\emph{Val+Proxy against Val}} \\
Lipo (mismatch) & -0.0005 & [-0.0070, 0.0000] & 11/2/7 & \textbf{0.046} & -0.0004 & [-0.0044, 0.0000] & 10/1/9 & \textbf{0.038} \\
Caco2 (scaffold) & +0.0000 & [-0.0218, 0.0000] & 7/3/10 & \textbf{0.047} & +0.0000 & [-0.0218, 0.0000] & 7/3/10 & 0.094 \\
Amylase & -0.3863 & [-0.5346, 0.0000] & 13/2/5 & \textbf{0.027} & +0.0000 & [-0.1342, 0.0000] & 8/2/10 & 0.094 \\
Hydrophobic Core & +0.7251 & [+0.3022, +1.1516] & 5/22/3 & \textbf{$<$0.001} & +0.5885 & [+0.0532, +0.8273] & 5/20/5 & \textbf{0.002} \\
\bottomrule\end{tabular}\end{table}

\begin{table}[htbp]\centering\scriptsize\setlength{\tabcolsep}{3pt}
\caption{Metric robustness on the audited pools: selectors that beat Val at Holm-adjusted $\alpha=0.05$ under each metric, with the selection rule unchanged. MSE is train-$\sigma$ standardised; MAE is what the TDC leaderboards report; $\rho$ is Spearman, what FLIP2 reports.}
\label{tab:app-metric}
\begin{tabular}{llll}\toprule
condition & under MSE & under MAE & under $\rho$ \\\midrule
Caco2 (random) & -- & -- & -- \\
Caco2 (scaffold) & -- & -- & -- \\
Lipo (random) & -- & -- & -- \\
Lipo (scaffold) & -- & -- & -- \\
Lipo (mismatch) & Val+Proxy, Val+Proxy (pen.) & -- & -- \\
Amylase & -- & -- & -- \\
Hydrophobic Core & \#params & Train, \#params & \#params \\
GDSC2 & -- & -- & -- \\
\bottomrule\end{tabular}\end{table}

\section{Sequential hyperparameter optimisation}
\label{app:seq}

\paragraph{Arms.} Nine per condition: random, TPE and HEBO on validation alone, and
matched multi-objective TPE and HEBO on each of the three training-side signals. The
TPE arms use Optuna's MOTPE with non-dominated sorting and a hypervolume tie-break;
the HEBO arms use its general optimiser configured for two objectives. Neither
scalarises, so no relative weight between validation loss and the signal had to be
chosen, which matters because choosing one would itself have required validation.
The penultimate-layer arms, and all conditions other than Lipophilicity mismatch and
Hydrophobic Core, were added after unblinding (Appendix~\ref{app:protocol}).

\paragraph{Deploy rule.} A multi-objective optimiser returns a Pareto front rather than
a single recommendation, so a rule to collapse it is unavoidable. Single-objective arms
deploy the trial minimising validation MSE, which is the optimiser's own answer.
Multi-objective arms deploy the trial minimising the rank sum of validation MSE and the
signal over all observed trials - the same rule used in the shared pool, which is what
makes the two phases comparable. The rank sum runs over every observed trial, not only
the Pareto-optimal subset, so these arms measure multi-objective search combined with
our deploy rule rather than what the optimiser would itself recommend.

\paragraph{Results.} Table~\ref{tab:app-seq} gives all nine arms on all eight
conditions; it and Figure~\ref{fig:app-budget} use unfiltered search trajectories.
Figure~\ref{fig:app-budget} shows the budget dependence: expanding the
trajectory from 8 to 48 trials improves deployment where validation rankings already
transfer, and does not on Hydrophobic Core. Optimising a misaligned objective more
thoroughly does not reveal the extrapolative optimum. These conclusions concern
multi-objective search combined with our deploy rule; other ways of using geometry
inside a search were not tested.

\paragraph{The positive case.} The Lipophilicity-scaffold improvement is
documented here in full, together with the two conditions that fail to reproduce it -
the matched random split and Caco2 scaffold - and the held-out-compound condition where
the same signal family is harmful. Lipophilicity scaffold was pre-registered as
descriptive rather than confirmatory (Appendix~\ref{app:protocol}), so this is a
boundary case and not evidence of a general advantage.

\paragraph{A caveat on pooling.} Statements pooled across arms are descriptive only.
The underlying observations are method~$\times$~replication pairs and are not
independent: the nine arms within one replication share a split, an initialisation seed
and a data ordering. All inference is therefore per condition, paired within
replication, with Holm correction inside that condition.

\begin{figure}[htbp]
\centering
\includegraphics[width=\textwidth]{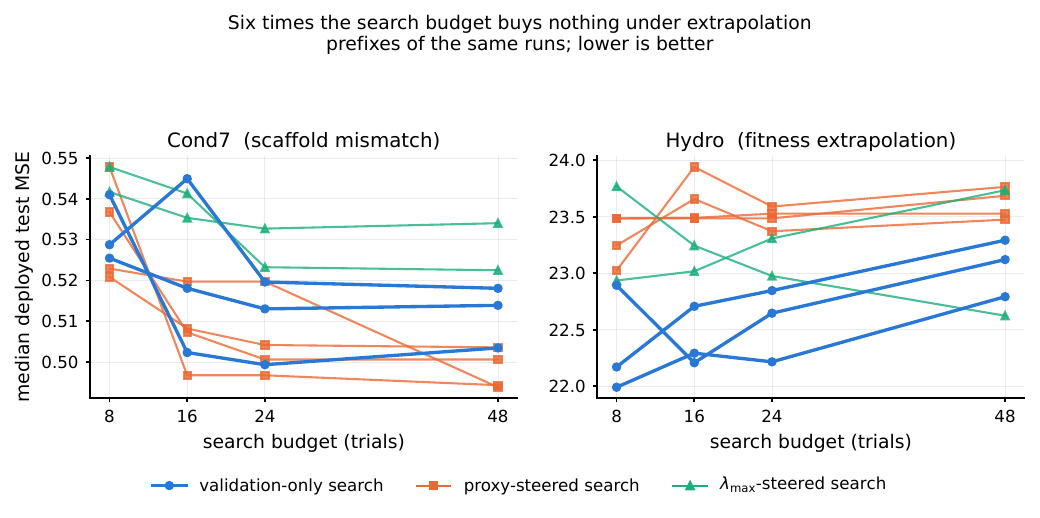}
\caption{Deployment loss against search budget. More search helps where validation
ranking transfers and does not where it fails. Unfiltered search trajectories; Cond7 is
Lipophilicity (mismatch) and Hydro is Hydrophobic Core.}
\label{fig:app-budget}
\end{figure}

\begin{table}[htbp]\centering\scriptsize\setlength{\tabcolsep}{3pt}
\caption{Sequential search, all nine arms, deployment MSE (median over replications, unfiltered trajectories). The reference is validation-only TPE; \textbf{bold} marks arms that Holm-beat it within that condition, \underline{underline} marks arms that are Holm-worse. Holm is applied over the eight contrasts within each condition. Multi-objective arms deploy by the rank-sum rule of Appendix~\ref{app:protocol}.}
\label{tab:app-seq}
\begin{tabular}{lrrrrrrrrr}\toprule
& \multicolumn{3}{c}{Val} & \multicolumn{2}{c}{Val+Proxy} & \multicolumn{2}{c}{Val+Proxy (pen.)} & \multicolumn{2}{c}{Val+$\lambda_{\max}$} \\
\cmidrule(lr){2-4}\cmidrule(lr){5-6}\cmidrule(lr){7-8}\cmidrule(lr){9-10}
condition & Random & TPE & HEBO & TPE & HEBO & TPE & HEBO & TPE & HEBO \\\midrule
Lipo (mismatch) & 0.518 & 0.503 & 0.514 & 0.501 & 0.494 & 0.494 & 0.504 & 0.522 & \underline{0.534} \\
Hydrophobic Core & 22.79 & 23.29 & 23.12 & 23.76 & 23.69 & 23.47 & 23.53 & 23.74 & 22.62 \\
Caco2 (random) & 0.351 & 0.355 & 0.351 & 0.340 & 0.341 & 0.332 & 0.335 & 0.344 & 0.342 \\
Caco2 (scaffold) & 0.464 & 0.481 & 0.460 & 0.443 & 0.457 & 0.424 & 0.445 & 0.464 & 0.442 \\
Lipo (random) & 0.412 & 0.409 & 0.424 & 0.411 & 0.411 & 0.415 & 0.408 & 0.407 & 0.416 \\
Lipo (scaffold) & 0.522 & 0.521 & 0.536 & \textbf{0.501} & \textbf{0.514} & \textbf{0.501} & \textbf{0.516} & 0.542 & 0.525 \\
Amylase & 2.374 & 2.465 & 2.674 & \textbf{1.845} & 2.065 & \textbf{1.757} & \textbf{1.954} & 2.135 & 2.025 \\
GDSC2 & 0.739 & 0.727 & 0.722 & \underline{0.767} & 0.756 & 0.765 & 0.750 & 0.724 & 0.745 \\
\bottomrule\end{tabular}\end{table}

\section{Degeneracy audit}
\label{app:degen}

Low activity drives every tracked signal towards its minimum. A collapsed network has
a near-zero proxy, and its $\lambda_{\max}$ sits at the floor of 2 set by the output
bias (Appendix~\ref{app:protocol}). Across all conditions, constant predictors have a
median $\lambda_{\max}$ of $2.14$ against $16.3$ for other networks, and a median proxy
of $1.9\times10^{-5}$ against $0.015$. Whether this matters is a property of the
condition, not of the signal, and must be checked per condition rather than assumed.

On Amylase it matters completely. Before filtering, the proxy improves deployment MSE
by about $1.39$ and wins all 20 replications. The lower error is real, but it is not
useful variant ranking. The model the proxy selects has a median dead-unit fraction of
$0.998$ and effectively zero prediction variance. Its median deployment MSE, $1.145$,
equals that of a predictor that outputs the training mean, and its Spearman correlation
is undefined in 14 of 20 replications because its predictions are exactly constant.
Under this position shift, predicting close to the training mean is nearly optimal in
MSE: the within-pool oracle reaches $1.13$, while the validation-selected model reaches
$2.54$. Collapse therefore presents as successful out-of-distribution selection. FLIP2
independently reports poor supervised performance on this split while zero-shot
protein-language-model likelihoods do substantially better \citep{didi2026flip2},
consistent with the split rewarding a degenerate solution rather than with our optimiser
failing to train. Table~\ref{tab:app-collapse} shows the same behaviour for every
criterion that deploys the lowest geometric score, and for none that combines a
geometric score with validation.

\paragraph{What the audit removes.} The audit's effect on Amylase comes from its
dead-unit criterion, not from removing constant predictors
(Table~\ref{tab:app-collapse-filters}). Removing only the candidates flagged as constant
leaves the proxy's advantage intact: the proxy then deploys a nearly inactive
network, with a median of 96\% dead units and a test-prediction spread near $10^{-5}$,
which is constant in all but name. Excluding networks with more than half their units
dead, with or without the constant criterion, shrinks the difference to a
non-significant $-0.36$ ($p_{\mathrm{Holm}}=0.108$).

Hydrophobic Core is the instructive intermediate case: many partially dead networks but
no constant predictors, which is why it is reported rather than excluded. No selector
deploys a constant predictor there, and a training-mean predictor reaches an MSE of
$29.8$, worse than every selector in Table~\ref{tab:app-pool-hydro}, so collapse is not
rewarded. Table~\ref{tab:app-degen} gives the unfiltered quantities per condition, and
Figure~\ref{fig:app-degen} shows the mechanism directly.

The general lesson is methodological, and it is the part of this study we would most
want another group to adopt: any selection criterion rewarding low activity, low
complexity or low curvature needs an explicit collapse audit before its apparent gains
are interpreted. Without one, ``my signal works'' and ``my signal found a dead network''
are indistinguishable.

\begin{figure}[htbp]
\centering
\includegraphics[width=\textwidth]{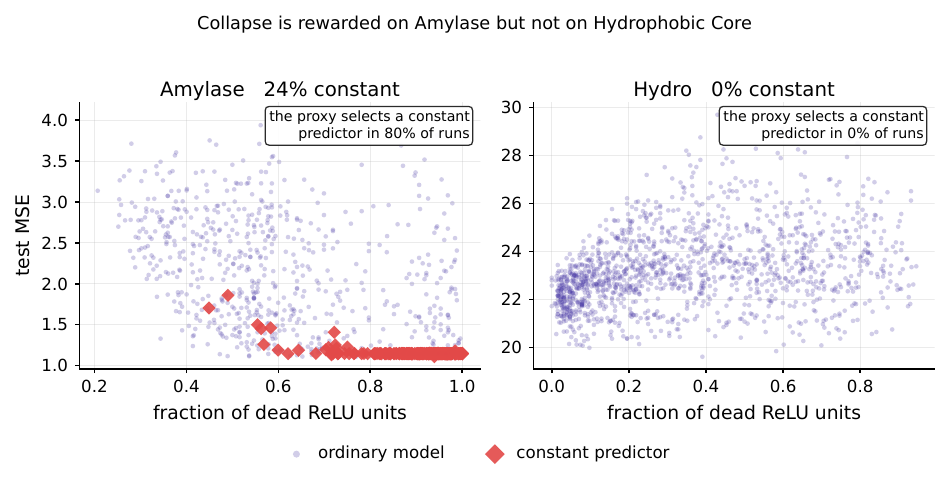}
\caption{Dead-unit fraction against deployment loss on the unfiltered pools, with
constant predictors marked. The two panels are a contrast, not a pattern. On Amylase
(left) $24\%$ of candidates collapse to a constant predictor, and because predicting
near the training mean is competitive in MSE on that split, they sit among the
lowest-loss models, so the proxy selects one in $80\%$ of replications. On Hydrophobic
Core (right) many networks are partially dead but none is constant, and the proxy never
selects one. Rates for every selector are in Table~\ref{tab:app-collapse}. Whether
minimising a geometric signal rewards collapse is therefore a property of the
condition, not of the signal, which is why it has to be audited per condition rather
than assumed.}
\label{fig:app-degen}
\end{figure}

\begin{table}[htbp]\centering\scriptsize\setlength{\tabcolsep}{3pt}
\caption{Degeneracy by condition, computed \emph{before} filtering. ``const.'' is the fraction of candidates whose test-set predictions have $\mathrm{std}(\hat y)<10^{-6}$; ``dead$>$0.5'' the fraction with more than half their ReLU units inactive on every example of the training proxy subset; ``dead@Proxy'' the median dead-unit fraction of the model Proxy selects.}
\label{tab:app-degen}
\begin{tabular}{lrrrr}\toprule
condition & const. & dead$>$0.5 & median dead & dead@Proxy \\\midrule
Caco2 (random) & 0.0\% & 0.0\% & 0.003 & \textbf{0.026} \\
Caco2 (scaffold) & 0.0\% & 0.0\% & 0.002 & \textbf{0.028} \\
Lipo (random) & 0.0\% & 2.8\% & 0.008 & \textbf{0.098} \\
Lipo (scaffold) & 0.0\% & 2.5\% & 0.008 & \textbf{0.095} \\
Lipo (mismatch) & 0.0\% & 2.1\% & 0.008 & \textbf{0.069} \\
Amylase & 24.3\% & 76.5\% & 0.675 & \textbf{0.998} \\
Hydrophobic Core & 0.0\% & 30.1\% & 0.298 & \textbf{0.525} \\
GDSC2 & 0.0\% & 0.4\% & 0.004 & \textbf{0.051} \\
\bottomrule\end{tabular}\end{table}

\begin{table}[htbp]\centering\scriptsize\setlength{\tabcolsep}{4pt}
\caption{How often each selector deploys a constant predictor on the unfiltered Amylase pools (20 replications). A model counts as constant when the standard deviation of its predictions is below $10^{-6}$, on the test set (as in the audit) or on the training set. Spearman is undefined when the test predictions are exactly constant. No selector deploys a constant predictor on Hydrophobic Core.}
\label{tab:app-collapse}
\begin{tabular}{lrrrr}\toprule
selector & constant (test) & constant (train) & undefined Spearman & median dead fraction \\\midrule
Val & 0\% & 0\% & 0/20 & 0.482 \\
Proxy & 80\% & 100\% & 14/20 & 0.998 \\
Proxy (pen.) & 85\% & 100\% & 14/20 & 0.973 \\
$\lambda_{\max}$ & 100\% & 100\% & 16/20 & 0.984 \\
Val+Proxy & 0\% & 0\% & 0/20 & 0.502 \\
Val+Proxy (pen.) & 0\% & 0\% & 0/20 & 0.532 \\
Val+$\lambda_{\max}$ & 0\% & 0\% & 0/20 & 0.931 \\
Train & 0\% & 0\% & 0/20 & 0.370 \\
\#params & 10\% & 10\% & 1/20 & 0.758 \\
\bottomrule\end{tabular}\end{table}

\begin{table}[htbp]\centering\scriptsize\setlength{\tabcolsep}{4pt}
\caption{Proxy against Val on Amylase under each candidate filter, deployment MSE, Holm-corrected across the four confirmatory conditions. Only the dead-unit criterion makes the difference non-significant.}
\label{tab:app-collapse-filters}
\begin{tabular}{lrrr}\toprule
filter & $\Delta$ & w/l & $p_{\text{Holm}}$ \\\midrule
none (primary analysis) & -1.3934 & 20/0 & $<$0.001 \\
constant predictors, test-set flag & -1.3938 & 20/0 & $<$0.001 \\
constant predictors, training-set flag & -0.8218 & 19/1 & $<$0.001 \\
more than half the units dead & -0.3628 & 14/5 & 0.108 \\
full audit (both criteria) & -0.3628 & 14/5 & 0.108 \\
\bottomrule\end{tabular}\end{table}

\section{Hyperparameter mechanism}
\label{app:hpmech}

Table~\ref{tab:app-hp} compares, within each identical audited pool, the
hyperparameters of the selected model against those of the deployment oracle. On Hydrophobic Core the oracle
prefers a learning rate $10^{0.312}\simeq2.05\times$ larger and a width $2^{2}=4\times$
larger, with less dropout and less weight decay; both differences are consistent across
replications. No comparable systematic mismatch appears on the molecular conditions.

The same table explains the \#params result, which is otherwise easy to
misread. Selecting the smallest model beats validation on Hydrophobic Core under all
three metrics, but not because small models extrapolate better: \#params and
validation choose the \emph{same} width, and both are wrong about it in the same
direction, since the oracle prefers wider models. The heuristic wins because the models
it lands on are closer to the oracle on learning rate, dropout and weight decay, and
those values are effectively random: \#params deploys the first-drawn of the one to
five smallest networks in the pool (Appendix~\ref{app:protocol}). The lesson is not ``prefer small models'' but that a heuristic containing no
geometry can outperform validation simply by being less misaligned with deployment on
the axes that matter.

\begin{table}[htbp]\centering\scriptsize\setlength{\tabcolsep}{3pt}
\caption{Median paired hyperparameter difference between the selected model and the deployment oracle in the same pool. Positive $\Delta\log_{10}\eta$ means the oracle prefers a larger learning rate. Audited pools.}
\label{tab:app-hp}
\begin{tabular}{llrrrr}\toprule
condition & selector & $\Delta\log_{10}\eta$ & $\Delta\log_2 w$ & $\Delta$dropout & $\Delta$wd \\\midrule
Caco2 (random) & Val & +0.110 & +0.00 & -0.029 & +5.8e-04 \\
Caco2 (random) & \#params & +0.170 & +1.50 & -0.056 & +7.6e-04 \\
Caco2 (scaffold) & Val & +0.222 & +0.50 & +0.014 & -1.3e-06 \\
Caco2 (scaffold) & \#params & +0.171 & +2.00 & +0.112 & +2.0e-03 \\
Lipo (random) & Val & +0.000 & +0.00 & +0.013 & +0.0e+00 \\
Lipo (random) & \#params & -0.227 & +3.00 & +0.044 & +9.7e-04 \\
Lipo (scaffold) & Val & +0.000 & +0.00 & +0.000 & +2.5e-04 \\
Lipo (scaffold) & \#params & -0.111 & +3.00 & +0.048 & +2.9e-03 \\
Lipo (mismatch) & Val & +0.119 & +0.00 & +0.013 & +3.1e-04 \\
Lipo (mismatch) & \#params & -0.051 & +2.50 & -0.026 & +1.1e-03 \\
Amylase & Val & -0.117 & -1.00 & +0.114 & -3.9e-06 \\
Amylase & \#params & -0.126 & +1.00 & +0.189 & +7.9e-07 \\
Hydrophobic Core & Val & +0.312 & +2.00 & -0.218 & -4.7e-05 \\
Hydrophobic Core & \#params & +0.187 & +2.00 & -0.158 & -1.5e-06 \\
GDSC2 & Val & +0.089 & -1.00 & -0.061 & +1.1e-05 \\
GDSC2 & \#params & +0.066 & +1.00 & -0.111 & -2.6e-05 \\
\bottomrule\end{tabular}\end{table}

\section{Geometry and curvature diagnostics}
\label{app:geom}

Table~\ref{tab:app-geom} gives the marginal and partial correlations behind
Section~\ref{sec:failure}. The proxy and $\lambda_{\max}$ appear strongly
anti-correlated marginally, but the association is largely induced by the learning
rate, which raises the proxy while lowering $\lambda_{\max}$; controlling for it
collapses the relationship. Correlations computed within narrow learning-rate bins agree
with the residualised estimates, so this is not an artifact of linear residualisation.
Throughout, $\lambda_{\max}$ and $\operatorname{tr}(H)$ are full-network quantities
(Appendix~\ref{app:protocol}), not the final-layer Hessian bounded in
Proposition~\ref{prop:bound}. The proxy's partial correlation with the gap is positive on
the molecular conditions and Amylase, negative on Hydrophobic Core and near zero on
GDSC2; where it is negative, a rule that deploys the lowest score is misdirected.

Table~\ref{tab:app-trace} addresses the natural rescue - that $\lambda_{\max}$ fails
because it summarises a single direction. The top eigenvalue is indeed a median
$0.12$--$0.14$ of the trace, so it does miss most of the curvature; measuring the rest
does not help. The Hutchinson estimator uses 30 Rademacher probes with a relative
standard error of $3.2$--$4.9\%$, small enough for the trace columns to carry the
comparison. The approximate relative-flatness quantity
$\lVert w\rVert^2\operatorname{tr}(H)$ is the one curvature-derived measure behaving as
flatness theory predicts on the molecular conditions, but decomposing it shows the
weight norm carries that behaviour and the trace contributes almost nothing once the
norm is controlled; on Hydrophobic Core the decomposition inverts. This is why we
describe the proxy as behaving like a capacity or norm quantity rather than a curvature
one. Our relative-flatness measure is a network-level surrogate, not the exact
layer-wise quadratic form of \citet{petzka2021relative}, and we do not claim otherwise.

\begin{figure}[htbp]
\centering
\includegraphics[width=\textwidth]{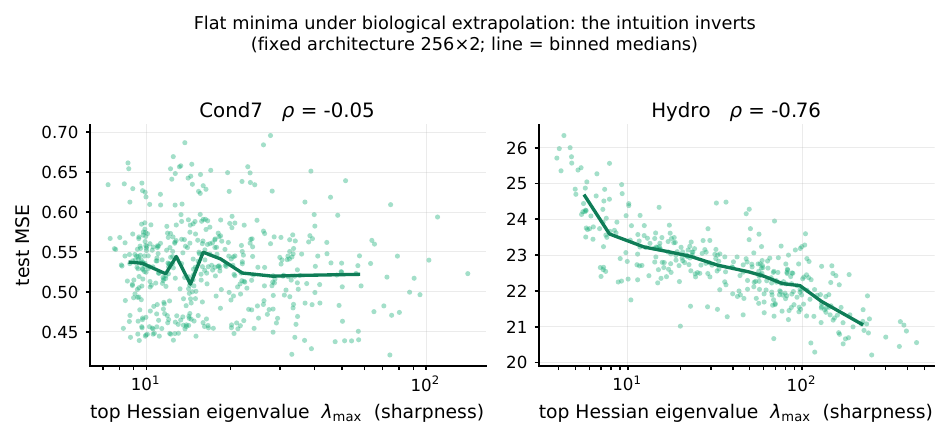}
\caption{Top Hessian eigenvalue against deployment MSE, on the fixed-architecture
subset (two hidden layers of width 256; line = binned medians). On the constructed
mismatch there is essentially no relationship ($\rho=-0.05$); under fitness
extrapolation the trend inverts strongly ($\rho=-0.76$), so sharper minima deploy
better. This is a different quantity from the trace--gap correlation reported in
\S\ref{sec:failure} ($\rho\approx-0.52$), which is measured on the full pool; the
corresponding $\lambda_{\max}$--gap correlation here is $-0.78$. Audited
fixed-architecture pools; Cond7 is Lipophilicity (mismatch) and Hydro is Hydrophobic Core.}
\label{fig:app-inv}
\end{figure}

\begin{table}[htbp]\centering\scriptsize\setlength{\tabcolsep}{3pt}
\caption{The proxy against exact curvature and against the generalisation gap. ``marg.'' is the marginal Spearman correlation across the sweep; ``partial'' additionally residualises on learning rate, width and depth. The learning rate is the confounder: it raises the proxy while lowering $\lambda_{\max}$, manufacturing the marginal association. Audited pools.}
\label{tab:app-geom}
\begin{tabular}{lrrrrrr}\toprule
 & \multicolumn{2}{c}{Proxy $\sim\lambda_{\max}$} & \multicolumn{2}{c}{Proxy $\sim$ gap} & \multicolumn{2}{c}{$\lambda_{\max}\sim$ gap} \\
\cmidrule(lr){2-3}\cmidrule(lr){4-5}\cmidrule(lr){6-7}
condition & marg. & partial & marg. & partial & marg. & partial \\\midrule
Caco2 (random) & -0.13 & +0.20 & +0.13 & +0.16 & +0.08 & +0.07 \\
Caco2 (scaffold) & -0.14 & +0.18 & +0.13 & +0.16 & +0.02 & -0.01 \\
Lipo (random) & -0.39 & -0.09 & +0.21 & +0.35 & +0.07 & -0.01 \\
Lipo (scaffold) & -0.39 & -0.09 & +0.15 & +0.23 & +0.04 & -0.01 \\
Lipo (mismatch) & -0.40 & -0.12 & +0.15 & +0.23 & +0.04 & -0.02 \\
Amylase & -0.58 & -0.73 & +0.27 & +0.56 & -0.31 & -0.36 \\
Hydrophobic Core & -0.55 & -0.14 & -0.23 & -0.30 & -0.19 & -0.48 \\
GDSC2 & +0.47 & +0.50 & -0.05 & +0.01 & -0.01 & -0.02 \\
\bottomrule\end{tabular}\end{table}

\begin{table}[htbp]\centering\scriptsize\setlength{\tabcolsep}{3pt}
\caption{Hessian trace and approximate relative flatness, $\lVert w\rVert^2\operatorname{tr}(H)$, on the three conditions where the trace was measured (720 networks trained, 642 after the audit; Hutchinson with 30 Rademacher probes). All entries are partial Spearman correlations with the generalisation gap given learning rate, width and depth. The final row shows that the trace adds almost nothing once the weight norm is controlled.}
\label{tab:app-trace}
\begin{tabular}{lrrr}\toprule
quantity & Caco2 (scaffold) & Lipo (mismatch) & Hydrophobic Core \\\midrule
$\lambda_{\max}$ & +0.01 & -0.05 & -0.48 \\
$\operatorname{tr}(H)$ & -0.02 & -0.17 & -0.52 \\
$\lVert w\rVert^2\operatorname{tr}(H)$ & +0.24 & +0.33 & -0.41 \\
$\lVert w\rVert^2$ alone & +0.18 & +0.31 & -0.21 \\
Proxy (ours) & +0.17 & +0.29 & -0.31 \\
$\operatorname{tr}(H)$ given $\lVert w\rVert^2$ & +0.08 & +0.07 & -0.49 \\
\midrule
median $\lambda_{\max}/\operatorname{tr}(H)$ & +0.139 & +0.117 & +0.136 \\
Hutchinson rel.\ SE & +0.044 & +0.032 & +0.049 \\
Proxy $\sim\operatorname{tr}(H)$ (partial) & -0.444 & -0.613 & -0.206 \\
Proxy $\sim\lVert w\rVert^2\operatorname{tr}(H)$ (partial) & +0.438 & +0.719 & +0.040 \\
\bottomrule\end{tabular}
\end{table}

\section{Ablations and additional experiments}
\label{app:abl}

Table~\ref{tab:app-abl} reports two ablations that rule out simpler explanations for
the failure.

\paragraph{Fixed architecture.} Pinning the network to two hidden layers of width 256
removes cross-architecture comparability as an explanation. It does not produce a
general rescue. On Hydrophobic Core under this restriction validation becomes
\emph{worse} than uniform random selection, in part because it selects models with
substantially more dead units than the oracle. The proximate mechanism therefore changes
with which degrees of freedom the search may vary, while the common failure does not:
the development regime rewards models that deployment does not.

\paragraph{Validation-free selection.} Folding the validation split back into training
removes validation from the problem entirely, isolating what the signal contributes on
its own rather than as a tie-break on top of a working validation set. In the shared
pool this costs nothing measurable on any non-degenerate condition. Descriptively,
adding the proxy to training-loss selection (Train+Proxy) gives a lower median
deployment MSE than training-loss selection alone in all four validation-free
conditions, and lower than random selection in three, all but Amylase
(Table~\ref{tab:app-abl}). We did not test these differences; they indicate that the
proxy carries some selection information, even though it does not consistently improve
on validation. In sequential search
the regime dependence is at its starkest: on Lipophilicity mismatch all six
signal-augmented arms beat training-MSE-only search after correction, while on
Hydrophobic Core none do and one is significantly worse. The sign flip therefore
survives the removal of validation and cannot be explained as an interaction between
the signal and the validation set.

\paragraph{Shift severity.} Table~\ref{tab:app-sev} increases feature extrapolation
continuously with a severity parameter $\alpha$ while holding validation
in-distribution, over ten synthetic conditions on two molecular datasets. Validation
rank transfer degrades with $\alpha$ as intended, but no crossover appears at which the
proxy reliably overtakes it. The real biological shifts also extend past the severity
the synthetic sweep reaches: Hydrophobic Core has essentially zero rank transfer and the
proxy still fails there. The useful negative statement is that severity of validation
failure does not predict when geometry helps.

\paragraph{Hard-example negative control.} To separate ``deployment is harder'' from
``deployment prefers a different model'', we built an adversarial split from the
held-out-compound condition using stored per-example predictions: examples all trained
models found difficult were moved toward deployment and easier ones toward validation,
with the explicit intent of manufacturing validation failure. It did the opposite. Rank
transfer rose to $+0.83$ and $+0.71$ on two replications, higher than the unmodified
condition's $+0.19$. All candidates find broadly the same examples hard, so relocating
them lowers every model's score by a similar amount and leaves the ordering intact - it
even stabilises it, because the easy validation subset is less noisy. This construction
is deliberately circular and cannot support any positive claim about performance; it is
usable only for the negative conclusion it was built to test, namely that prediction
difficulty and model-ranking shift are separate phenomena and only the second breaks
selection.

\begin{table}[htbp]\centering\small\setlength{\tabcolsep}{4pt}
\caption{Ablations, deployment MSE (median over replications, audited pools). \emph{Fixed architecture} pins the network to two hidden layers of width 256; its Proxy and $\lambda_{\max}$ columns select on the signal alone. \emph{Validation-free} folds the validation split back into training, so only train-side selectors are defined, Train replaces Val as the reference, and the Proxy and $\lambda_{\max}$ columns are rank sums with training MSE (Train+Proxy and Train+$\lambda_{\max}$).}
\label{tab:app-abl}
\begin{tabular}{lllrrrr}\toprule
ablation & condition & reference & ref. & Proxy & $\lambda_{\max}$ & Random \\\midrule
fixed arch. & Lipo (mismatch) & Val & 0.510 & 0.512 & 0.543 & 0.533 \\
fixed arch. & Hydrophobic Core & Val & 23.245 & 22.565 & 25.532 & 22.654 \\
validation-free & Caco2 (scaffold) & Train & 0.476 & 0.459 & 0.482 & 0.496 \\
validation-free & Lipo (scaffold) & Train & 0.543 & 0.507 & 0.528 & 0.540 \\
validation-free & Amylase & Train & 3.182 & 3.015 & 3.117 & 2.847 \\
validation-free & Hydrophobic Core & Train & 22.871 & 22.837 & 23.121 & 23.272 \\
\bottomrule\end{tabular}\end{table}

\begin{table}[htbp]\centering\small\setlength{\tabcolsep}{4pt}
\caption{Synthetic shift-severity sweep. $\alpha=0$ is an IID split and $\alpha=1$ is pure feature extrapolation; validation is held in-distribution throughout. $\rho_{\text{val}}$ degrades with severity, but no threshold appears at which the proxy reliably overtakes validation. Audited pools.}
\label{tab:app-sev}
\begin{tabular}{lrrrr}\toprule
condition & $\rho_{\text{val}}$ & Val & Proxy & Random \\\midrule
Caco2 $\alpha=0.00$ & +0.52 & 0.354 & 0.342 & 0.367 \\
Caco2 $\alpha=0.25$ & +0.60 & 0.339 & 0.347 & 0.380 \\
Caco2 $\alpha=0.50$ & +0.51 & 0.451 & 0.401 & 0.448 \\
Caco2 $\alpha=0.75$ & +0.43 & 0.710 & 0.668 & 0.792 \\
Caco2 $\alpha=1.00$ & +0.40 & 0.834 & 0.755 & 0.913 \\
Lipo $\alpha=0.00$ & +0.74 & 0.400 & 0.407 & 0.435 \\
Lipo $\alpha=0.25$ & +0.77 & 0.384 & 0.381 & 0.421 \\
Lipo $\alpha=0.50$ & +0.79 & 0.436 & 0.406 & 0.439 \\
Lipo $\alpha=0.75$ & +0.66 & 0.463 & 0.473 & 0.503 \\
Lipo $\alpha=1.00$ & +0.64 & 0.593 & 0.564 & 0.593 \\
\bottomrule\end{tabular}\end{table}

\section{Benchmark sanity checks}
\label{app:sanity}

A negative result about model selection is only interesting if the models being
selected among are competent, and if the comparison to published work is like-for-like.

\paragraph{TDC molecular tasks.} Our validation-deployed models are below
the TDC leaderboards on the matched scaffold splits: Caco2-Wang MAE $0.416$ against a
leaderboard best near $0.256$, and Lipophilicity $0.663$ against $0.456$. These are medians over replications;
Table~\ref{tab:all_results} reports means of the same validation-selected models
($0.432$ and $0.656$). We state the gap plainly, and note that it follows from
decisions made to reduce confounding: 2048-bit Morgan fingerprints rather than pretrained or
message-passing representations, a single MLP rather than boosted-tree ensembles, 100
fixed epochs with no early stopping, a fixed 48-point hypercube shared across selectors
rather than a task-tuned search, and no per-task feature engineering - all of which
follow from requiring every selector to rank an identical pool. It does not invalidate a
controlled selection study, because the question is which candidate a criterion picks
from a fixed pool, and the pools demonstrably contain models of meaningfully different
deployment quality.

\paragraph{Training budget.} Every candidate trains for a fixed 100 epochs without early
stopping, so that all selectors rank the same models. We do not claim convergence:
training MSE fell by a median of 13\% between epochs 50 and 100 (interquartile range
2--30\%), and by more than 10\% for 56\% of candidates. The comparisons in this appendix
show that the pools contain competitive models, not that every candidate converged.

\paragraph{FLIP2.} On Hydrophobic Core our one-hot MLP
exceeds the published benchmark baselines, so the extrapolation result cannot be
dismissed as an artifact of a uniformly weak predictor. On Amylase our supervised
performance is poor, which matches what the benchmark itself reports for supervised
models on that split; Appendix~\ref{app:degen} analyses its consequences for selection.

\paragraph{GDSC2.} Our RMSE of $1.535$ in $\ln(\mathrm{IC}_{50})$ units falls inside the
published leave-drug-out range.

\end{document}